\documentclass{article}

    \PassOptionsToPackage{numbers, compress}{natbib}

\usepackage{subfigure}

\usepackage[final]{neurips_2021}

\usepackage[utf8]{inputenc} 
\usepackage[T1]{fontenc}    
\usepackage[colorlinks,
            linkcolor=red,
            anchorcolor=blue,
            citecolor=blue
            ]{hyperref}
\usepackage{url}            
\usepackage{booktabs}       
\usepackage{amsfonts}       
\usepackage{nicefrac}       
\usepackage{microtype}      
\usepackage{xcolor}         

\usepackage{enumerate}
\usepackage[OT1]{fontenc}
\usepackage{mylatexstyle}
\usepackage{mathrsfs,bbm}

\def\sgdbias{\mathrm{SGDBias}}
\def\ridgebias{\mathrm{RidgeBias}}
\def\sgdvar{\mathrm{SGDVariance}}
\def\ridgevar{\mathrm{RidgeVariance}}

\def \Nr {N_{\mathrm{ridge}}}
\def \Ns {N_{\mathrm{sgd}}}

\usepackage[compact]{titlesec}
\usepackage{sidecap}
\titlespacing*{\section}{0pt}{*0.1}{*0.1}
\titlespacing*{\subsection}{0pt}{*0.1}{*0.1}

\newcommand{\qg}[1]{\textsf{\color{orange} QG: #1}}

\title{The Benefits of Implicit Regularization \\
from SGD in Least Squares Problems}

\author
{
Difan Zou\thanks{Equal Contribution} \\
University of California, Los Angeles\\
\texttt{ knowzou@cs.ucla.edu}
\And
	Jingfeng Wu$^*$\\
	Johns Hopkins University\\
	\texttt{uuujf@jhu.edu} 
\And
	Vladimir Braverman\\
	Johns Hopkins University\\
	\texttt{
vova@cs.jhu.edu}
\And
	Quanquan Gu\\ 
	 University of California, Los Angeles\\
	\texttt{ qgu@cs.ucla.edu}
\And
	Dean P. Foster\\
	Amazon\\
	\texttt{dean@foster.net}
\And
	Sham M. Kakade \\ University of Washington \& Microsoft Research\\
	\texttt{ sham@cs.washington.edu}
}

\begin{document}

\maketitle


\begin{abstract}

Stochastic gradient descent (SGD) exhibits strong algorithmic regularization effects in practice, which has been hypothesized to play an important role in the generalization of modern machine learning approaches. In this work, we seek to understand these issues in the simpler setting of linear regression (including both underparameterized and overparameterized regimes), where our goal is to make sharp instance-based comparisons of the implicit regularization afforded by (unregularized) average SGD with the explicit regularization of ridge regression.
For a broad class of least squares problem instances (that are natural in high-dimensional settings), we show:
(1) for every problem instance and for every ridge parameter, (unregularized) SGD, when provided with \emph{logarithmically} more samples than that provided to the ridge algorithm, generalizes no worse than the ridge solution (provided SGD uses a tuned constant stepsize);
(2) conversely, there exist instances (in this wide problem class) where optimally-tuned ridge regression requires \emph{quadratically} more samples than SGD in order to have the same generalization performance. Taken together, our results show that, up to the logarithmic factors, the generalization performance of SGD is always no worse than that of ridge regression in a wide range of overparameterized problems, and, in fact, could be much better for some problem instances. More generally, our results show how algorithmic regularization has important consequences even in simpler (overparameterized) convex settings.
\end{abstract}

\section{Introduction}

Deep neural networks often exhibit powerful generalization in numerous machine learning applications, despite being \emph{overparameterized}. It has been conjectured that the optimization algorithm itself, e.g., \emph{stochastic gradient descent} (SGD), implicitly regularizes such overparameterized models~\citep{zhang2016understanding}; here,  (unregularized) overparameterized models could admit numerous global and local minima (many of which generalize poorly~\citep{zhang2016understanding,liu2019bad}), yet SGD tends to find solutions that generalize well, even  in the absence of explicit regularizers~\citep{neyshabur2014search,zhang2016understanding,keskar2016large}.
This regularizing effect due to the choice of the optimization
algorithm is often referred to as \emph{implicit regularization}~\citep{neyshabur2014search}.

Before moving to the non-convex regime, we may hope to start by
understanding this effect in the (overparameterized) convex regime.  At least for linear
models, there is a
growing body of evidence suggesting that the implicit regularization
of SGD is closely related to an explicit, $\ell_2$-type of (ridge)
regularization~\citep{tihonov1963solution}.  For example, 
(multi-pass) SGD for linear regression converges to the \emph{minimum-norm interpolator},
which corresponds to the limit of the ridge solution with a vanishing
penalty~\citep{zhang2016understanding,gunasekar2018characterizing}.
Tangential evidence for this also comes from examining gradient
descent, where a continuous time (gradient flow) analysis shows how the
optimization path of gradient descent is (pointwise)
closely connected to an explicit, $\ell_2$-regularization~\citep{suggala2018connecting,ali2019continuous}. Similar results \citep{ali2020implicit} have been further extended to SGD, where a (early-stopped) continuous-time SGD is demonstrated to perform similarly to ridge regression with certain regularization parameters. 

However, as of yet, a precise comparison between the implicit regularization
afforded by SGD and the explicit regularization of ridge regression (in
terms of the \emph{generalization performance}) is still lacking, especially when the hyperparameters (e.g., stepsize for SGD and regularization parameter for ridge regression) are allowed to be tuned.
This motivates the central question in this work:
\begin{center}
    \emph{How does the generalization performance of SGD compare with that of ridge regression in least square problems?}
\end{center}
In particular, even in the arguably simplest
setting of linear regression, we seek to understand if/how SGD behaves differently from
using an explicit $\ell_2$-regularizer, with a particular focus
on the overparameterized regime.

\paragraph{Our Contributions.} Due to recent advances on sharp,
\emph{instance-dependent} excess risks bounds of both (single-pass) SGD
and ridge regression for overparameterized least square problems
\citep{tsigler2020benign,zou2021benign}, a nearly complete answer to
the above question is now possible using these tools.  In this work,
we deliver an \emph{instance-based} risk comparison between SGD and
ridge regression in several interesting settings, including one-hot
distributed data and Gaussian data.  In particular, for a broad class
of least squares problem instances that are natural in
high-dimensional settings, we show that
\begin{itemize}[leftmargin=*, nosep]
    \item For every problem instance and for every ridge parameter, (unregularized) SGD, when provided with \emph{logarithmically} more samples than that provided to ridge regularization, generalizes no worse than the ridge solution, provided SGD uses a tuned constant stepsize.
    \item Conversely, there exist instances in our problem class where optimally-tuned ridge regression requires \emph{quadratically} more samples than SGD  to achieve the same generalization performance.
\end{itemize}
 
Quite strikingly, the above results show that, up to some logarithmic
factors, the generalization performance of SGD is always no worse than
that of ridge regression in a wide range of overparameterized least
square problems, and, in fact, could be much better for some problem
instances. As a special case (for the above two claims), our problem class
includes a setting in which: (i) the
signal-to-noise is bounded and  (ii) the eigenspectrum decays at a
polynomial rate $1/i^\alpha$, for $0\leq\alpha\leq 1$ (which permits a
relatively fast decay).
This one-sided near-domination phenomenon (in these natural
overparameterized problem classes) could further support the preference for the implicit
regularization brought by SGD over  explicit
ridge regularization.

Several novel technical contributions are  made to make the above risk comparisons possible. For the one-hot data, we derive similar risk upper bound of SGD and risk lower bound of ridge regression.
For the Gaussian data, while a sharp risk bound of SGD is borrowed
from \citep{zou2021benign}, we prove a sharp lower bound of ridge
regression by adapting the proof techniques developed in
\citep{tsigler2020benign,bartlett2020benign}. 
By carefully comparing these upper and lower bound results (and
exhibiting particular instances to show that our sample size inflation
bounds are sharp), we are able to provide nearly complete conditions that characterize when SGD generalizes better than ridge regression.

\paragraph{Notation.}
For two functions $f(x) \ge 0$ and $g(x) \ge 0$ defined on $x > 0$,
we write $f(x) \lesssim g(x)$ if $f(x) \le c\cdot g(x)$ for some absolute constant $c > 0$; 
we write $f(x) \gtrsim g(x)$ if $g(x) \lesssim f(x)$;
we write  $f(x) \eqsim g(x)$ if $f(x) \lesssim g(x)\lesssim f(x)$.
For a vector $\wb \in \RR^d$ and a positive semidefinite matrix $\Hb \in \RR^{d\times d}$, we denote $\norm{\wb}_{\Hb} := \sqrt{ \wb^\top \Hb \wb }$.


\section{Related Work}

In terms of making sharp risk comparisons with ridge, the work of
\citep{dhillon2013risk} shows that OLS (after a PCA projection is applied to the data) is instance-wise competitive with ridge on fixed design problems. The insights in our analysis are draw from this work, though there are a number of technical challenges in dealing with the random design setting.  We start with a brief discussion of the technical advances in the analysis of ridge regression and SGD, and then briefly overview more related work comparing SGD to explicit norm-based regularization.
\paragraph{Excess Risk Bounds for Ridge Regression.}
In the underparameterized regime, the excess risk bounds for ridge regression has been well-understood~\citep{hsu2012random}. In the overparameterized regime, a large body of works \citep{dobriban2018high,hastie2019surprises,xu2019number,wu2020optimal} focused on characterizing the excess risk of ridge regression in the asymptotic regime where both the sample size $N$ and dimension $d$ go to infinite and $d/N\rightarrow \gamma$ for some finite $\gamma$.
More recently, \citet{bartlett2020benign} developed sharp non-asymptotic risk bounds for ordinary least square in the overparameterized setting, which are further extended to ridge regression by \citet{tsigler2020benign}. These bounds have additional interest because they are instance-dependent, in particular, depending on the data covariance spectrum. The risk bounds of ridge regression derived in \citet{tsigler2020benign} is highly nontrivial in the overparameterized setting as it holds when the ridge parameter equals to zero or even being negative. This line of results build one part of the theoretical tools for this paper. 

\paragraph{Excess Risk Bounds for SGD.}
Risk bounds for one-pass, constant-stepsize (average) SGD have been derived in the finite dimensional case \citep{bach2013non,defossez2015averaged,jain2017markov,jain2017parallelizing,dieuleveut2017harder,ali2019continuous}. 
Very recently, the work of \citep{zou2021benign} extends these analyses, providing
 sharp \emph{instance-dependent} risk bound applicable to the overparameterized regime; here, \citet{zou2021benign} provides nearly matching upper and lower excess risk bounds for constant-stepsize SGD, which are sharply characterized in terms of the full eigenspectrum of the population covariance matrix. 
This result plays a pivotal role in our paper.

\paragraph{Implicit Regularization of SGD vs. Explicit Norm-based Regularization.}

For least square problems, multi-pass SGD converges to the minimum-norm solution~\citep{neyshabur2014search,zhang2016understanding,gunasekar2018characterizing}, which is widely cited as (one of) the implicit bias of SGD.
However, in more general settings, e.g., convex but non-linear models, a (distribution-independent) norm-based regularizer is no longer sufficient to characterize the optimization behavior of SGD \citep{arora2019implicit,dauber2020can,razin2020implicit}. 
Those discussions, however, exclude the possibility of \emph{hyperparameter tuning}, e.g., stepsize for SGD and penalty strength for ridge regression, and are not instance-based, either.
Our aim in this paper is to provide instance-based excess risk comparison between the optimally tuned (one-pass) SGD and the optimally tuned ridge regression. 

\section{Problem Setup and Preliminaries}

We seek to compare the generalization ability of SGD and ridge algorithms for \emph{least square problems}.
We use $\xb\in\cH$ to denote a feature vector in a (separable) Hilbert space $\cH$.
We use $d$ to refer to the dimensionality of $\cH$, where $d = \infty$ if $\cH$ is infinite-dimensional.
We use $y\in\RR$ to denote a response that is generated by 
\begin{align*}
y = \la\xb,\wb^*\ra + \xi,
\end{align*}
where $\wb^*\in\cH$ is an unknown true model parameter and $\xi\in\RR$ is the model noise.
The following regularity assumption is made throughout the paper.
\begin{assumption}[Well-specified noise]\label{assump:model_noise}
The second moment of $\xb$, denoted by $\Hb := \EE[\xb \xb^\top]$, is strictly positive definite and has finite trace.
The noise $\xi$ is independent of $\xb$ and satisfies
\begin{align*}
\EE[\xi] = 0, \quad\mbox{and}\quad\EE[\xi^2]=\sigma^2.
\end{align*}
\end{assumption}

In order to characterize the interplay between $\wb^*$ and $\Hb$ in the excess risk bound, we introduce:
\begin{gather*}
  {\Hb}_{0:k} := \textstyle{\sum_{i=1}^k}\lambda_i\vb_i\vb_i^\top,\quad\mbox{and}\quad
  {\Hb}_{k:\infty} := \textstyle{\sum_{i> k}}\lambda_i\vb_i\vb_i^\top,
\end{gather*}
where $\{\lambda_i\}_{i=1}^\infty$ are the eigenvalues of $\Hb$ sorted in
non-increasing order and $\vb_i$'s are the corresponding eigenvectors. 
Then we define
\[
\|\wb\|^2_{\Hb_{0:k}^{-1}}=
\sum_{ i\le
  k}\frac{(\vb_i^\top\wb)^2}{\lambda_i}, \quad
\|\wb\|^2_{{\Hb}_{k:\infty}}=\sum_{i> k}\lambda_i (\vb_i^\top\wb)^2.
\]

The least squares problem is to estimate the true parameter $\wb^*$.
Assumption \ref{assump:model_noise} implies that $\wb^*$ is the unique solution that minimizes the \emph{population risk}:
\begin{align}\label{eq:population_risk}
L(\wb^*) = \min_{\wb\in\cH}L(\wb),\quad \text{where}\ L(\wb):= \frac{1}{2}\EE_{(\xb,y)\sim\cD}\big[(y - \la\wb,\xb\ra)^2\big].
\end{align}
Moreover we have that $L(\wb^*) = \sigma^2$.
For an estimation $\wb$ found by some algorithm, e.g., SGD or ridge regression, its performance is measured by the \emph{excess risk}, $L(\wb) - L(\wb^*)$.




\paragraph{Constant-Stepsize SGD with Tail-Averaging.}
We consider the constant-stepsize SGD with tail-averaging~\citep{bach2013non,jain2017markov,jain2017parallelizing,zou2021benign}: at the $t$-th iteration, a fresh example $(\xb_t,y_t)$ is sampled independently from the data distribution, and SGD makes the following update on the current estimator $\wb_{t-1}\in \cH$, 
\begin{align*}
\wb_t = \wb_{t-1} + \gamma\cdot\big(y_t - \la\wb_{t-1},\xb_t\ra\big)\xb_t,\  t=1,2,\ldots, \qquad \wb_0 = 0,
\end{align*}
where $\gamma>0$ is a constant stepsize.
After $N$ iterations (which is also the number of samples observed), SGD outputs the tail-averaged iterates as the final estimator:
\begin{align*}
    \wb_{\mathrm{sgd}}(N;\gamma) := \frac{2}{N}\sum_{t=N/2}^{N-1}\wb_t.
\end{align*}
In the underparameterized setting ($d < N$), constant-stepsize SGD with tail-averaging is known for achieving minimax optimal rate for least squares~\citep{jain2017markov,jain2017parallelizing}. 
More recently, \citet{zou2021benign} investigate the performance of constant-stepsize SGD with tail-averaging in the overparameterized regime ($d > N$), and establish \emph{instance-dependent}, nearly-optimal excess risk bounds under mild assumptions on the data distribution.
Notably, results from \citep{zou2021benign} cover underparameterized cases ($d<N$) as well.

 
\paragraph{Ridge Regression.}
Given $N$ i.i.d. samples $\{(\xb_i,y_i)\}_{i=1}^N$, let us denote $\Xb:= [\xb_1,\dots,\xb_N]^\top\in\RR^{N\times d}$ and $\yb := [y_1,\dots,y_N]^\top\in\RR^d$.
Then ridge regression outputs the following estimator for the true parameter~\citep{tihonov1963solution}:
\begin{align}\label{eq:ridge_solution1}
\wb_{\mathrm{ridge}}(N;\lambda) := \arg\min_{\wb\in\cH} \|\Xb\wb - \yb\|_2^2 + \lambda\|\wb\|_2^2,
\end{align}
where $\lambda$ (which could possibly be negative) is a regularization parameter.
We remark that the ridge regression estimator takes the following two equivalent form:
\begin{align}\label{eq:ridge_solution2}
\wb_{\mathrm{ridge}}(N;\lambda) 
=(\Xb^\top\Xb+\lambda\Ib_{d})^{-1}\Xb^\top\yb = \Xb^\top(\Xb\Xb^\top+\lambda\Ib_{N})^{-1}\yb.
\end{align}
The first expression is useful in the classical, underparameterized setting ($d < N$) \citep{hsu2012random};  
and the second expression is more useful in the overparameterized setting ($d > N$) where the empirical covariance $\Xb^\top \Xb$ is usually not invertible~\citep{kobak2020optimal,tsigler2020benign}. 
As a final remark, when $\lambda = 0$, ridge estimator reduces to the \emph{ordinary least square estimator} (OLS)~\citep{friedman2001elements}. 


\paragraph{Generalizable Regime.}
In the following sections we will make instance-based risk comparisons between SGD and ridge regression.
To make the comparison meaningful, we focus on regime where SGD and ridge regression are ``generalizable'', i.e, the SGD and the ridge regression estimators, with the optimally-tuned hypeparameters, can achieve excess risk that is smaller than the optimal population risk, i.e., $\sigma^2$. The formal mathematical definition is as follows.

\begin{definition}[Generalizability]\label{def:generalizable} 
Consider an algorithm $\mathtt{Alg}$ and a least squares problem instance $\mathtt{P}$.
Let $\mathtt{Alg}(n,\btheta)$ be the output of the algorithm when provided with  $n$ i.i.d. samples from the problem instance $\mathtt{P}$, and a set of hyperparameters $\btheta$ (that could be a function on $n$). 
Then we say that the algorithm $\mathtt{Alg}$ with sample size $n$ and hyperparameters configuration $\btheta$ is \emph{generalizable} on problem instance $\mathtt{P}$, if
\begin{align*}
\EE_{\mathtt{Alg}, \mathtt{P}} [L\big(\mathtt{Alg}(n,\btheta)\big)]-L(\wb^*) \leq \sigma^2,
\end{align*}
where the expectation is over the randomness of $\mathtt{Alg}$ and data drawn from the problem instance $\mathtt{P}$.
\end{definition}
Clearly, the generalizable regime is defined by conditions on both the sample size, hyperparameter configuration, the problem instance, and the algorithm. For example, in the $d$-dimensional setting with $\|\wb^*\|_2=O(1)$, the ordinary least squares (OLS) solution (ridge regression with $\lambda=0$), i.e.,  $\wb_{\mathrm{ridge}}(N;0)$ has $\cO(d\sigma^2/N)$ excess risk, then we can say that the ridge regression with regularization parameter $\lambda=0$ and sample size $N=\omega(d)$ is in the generalizable regime on all problem instances in $d$-dimension with $\|\wb^*\|_2=O(1)$.

\paragraph{Sample Inflation vs. Risk Inflation Comparisons.}
This work characterizes the \emph{sample inflation} of SGD, i.e., bounding the required sample size of SGD to achieve an instance-based comparable excess risk as ridge regression
(which is essentially the notion of Bahadur statistical efficiency~\cite{Bahadur1967RatesOC,doi:10.1137/1.9781611970630}).
Another natural comparison would be examining the \emph{risk inflation} of SGD, examining the instance-based increase in risk for any fixed sample size.
Our preference for the former is due to the relative instability of the risk with respect to the sample size (in some cases, given a slightly different sample size, the risk could rapidly change.).

\section{Warm-Up: One-Hot Least Squares Problems}
Let us begin with a simpler data distribution, the \emph{one-hot} data distribution. 
(inspired by settings where the input distribution is sparse). 
In detail, assume each input vector $\xb$ is sampled from the set of natural basis $\{\eb_1,\eb_2,\dots,\eb_d\}$ according to the data distribution given by $\PP \{ \xb = \eb_i\} = \lambda_i$, where $0 < \lambda_i \le 1$ and $\sum_i \lambda_i = 1$.
The class of one-hot least square instances is completely characterized by
the following problem set: \[
\big\{ (\wb^*; \lambda_1, \cdots, \lambda_d) :\ \wb^* \in \cH,\ \textstyle{\sum_{i}} \lambda_i = 1, \ 1 \ge  \lambda_1 \ge \lambda_2 \ge \dots > 0 \big\}.
\]
Clearly the population data covariance matrix is $\Hb = \diag(\lambda_1, \dots, \lambda_d)$.
The next two theorems give an instance-based sample inflation comparisons for this problem class.





\begin{theorem}[Instance-wise comparison, one-hot data]\label{thm:comparison_onehot}
Let $\wb_{\mathrm{sgd}}(N; \gamma)$ and $\wb_{\mathrm{ridge}}(N; \lambda)$ be the solutions found by SGD and ridge regression when using $N$ training examples. 
Then for any one-hot least square problem instance such that the ridge regression solution is generalizable and any $\lambda$, there exists a choice of stepsize $\gamma^*$ for SGD such that
\begin{align*}
L\big[\wb_{\mathrm{sgd}}(N_{\mathrm{sgd}}; \gamma^*)\big]-L(\wb^*)\lesssim L\big[\wb_{\mathrm{ridge}}(N_{\mathrm{ridge}};\lambda)\big] -L(\wb^*) < \sigma^2,
\end{align*}
provided the sample size of SGD satisfies
\begin{align*}
\Ns \ge \Nr.
\end{align*}
\end{theorem}

Theorem \ref{thm:comparison_onehot} suggests that for \emph{every} one-hot problem instance, when provided with the same or more number of samples, the SGD solution with a properly tuned stepsize generalizes at most constant times worse than the optimally tuned ridge regression solution. 
In other words, with the same number of samples, SGD is \emph{always} competitive with ridge regression. 


\begin{theorem}[Best-case comparison, one-hot data]\label{thm:SGD>ridge_onehot}
There exists an one-hot least square problem instance satisfying $\|\wb^*\|_\Hb^2 = \sigma^2$, and a SGD solution with constant stepsize and sample size $N_{\mathrm{sgd}}$, such that for any ridge regression solution with sample size
\begin{align*}
\Nr \le \frac{\Ns^2}{\log^2(\Ns)},
\end{align*}
it holds that,
\begin{align*}
 L\big[\wb_{\mathrm{ridge}}(N_{\mathrm{ridge}};\lambda)\big]-L(\wb^*) \gtrsim L\big[\wb_{\mathrm{sgd}}(N_{\mathrm{sgd}}; \gamma^*)\big]-L(\wb^*).
\end{align*}
\end{theorem}

Theorem~\ref{thm:SGD>ridge_onehot} shows that for some one-hot least square instance, ridge regression, even with the optimally-tuned regularization, needs at least (nearly) quadratically more samples than that provided to SGD, in order to compete with the optimally-tuned SGD.
In other words, ridge regression could be much worse than SGD for one-hot least squares problems.

\begin{remark}
The above two results together indicate a \emph{superior} performance of the implicit regularization of SGD in comparison with the explicit regularization of ridge regression, for one-hot least squares problems. This is not the only case that SGD is always no worse than ridge estimator. In fact,
we will next turn to compare SGD with ridge regression for the class of Gaussian least square instances, where both SGD and ridge regression exhibit richer behaviors but SGD still exhibits superiority over the ridge estimator.
\end{remark}






\section{Gaussian Least Squares Problems}


In this section, we consider least squares problems with a Gaussian data distribution.
In particular, assume the population distribution of the input vector $\xb$ is Gaussian\footnote{We restrict ourselves to the Gaussian distribution for simplicity. Our results hold under more general assumptions, e.g., $\Hb^{-1/2}\xb$ has sub-Gaussian tail and independent components~\citep{bartlett2020benign} and is symmetrically distributed.}, i.e., $\xb \sim \cN(\boldsymbol{0},\Hb)$.
We further make the following regularity assumption for simplicity:
\begin{assumption}\label{assump:data_distribution}
$\Hb$ is strictly positive definite and has a finite trace.
\end{assumption}
Gaussian least squares problems are completely characterized by the following problem set $\big\{ (\wb^*; \Hb) :\ \wb^* \in \cH \big\}$.

The next theorem give an instance-based sample inflation comparison between SGD and ridge regression for Gaussian least squares instances.


\begin{theorem}[Instance-wise comparison, Gaussian data]\label{thm:SGD<ridge}
Let $\wb_{\mathrm{sgd}}(N; \gamma)$ and $\wb_{\mathrm{ridge}}(N; \lambda)$ be the solutions found by SGD and ridge regression respectively. Then under Assumption \ref{assump:data_distribution}, for any Gaussian least square problem instance such that the ridge regression solution is generalizable  and any $\lambda$, there exists a choice of stepsize $\gamma^*$ for SGD such that
\begin{align*}
L\big[\wb_{\mathrm{sgd}}(N_{\mathrm{sgd}}; \gamma^*)\big]-L(\wb^*)\lesssim L\big[\wb_{\mathrm{ridge}}(N_{\mathrm{ridge}};\lambda)\big] -L(\wb^*),
\end{align*}
provided the sample size of SGD satisfies
\begin{align*}
N_{\mathrm{sgd}} \ge (1+R^2) \cdot \kappa(\Nr) \cdot \log( a) \cdot \Nr,
\end{align*}
where
\[ 
\kappa(n)=\frac{\tr(\Hb)}{n\lambda_{\min\{n,d\}}},\quad 
R^2 = \frac{ \|\wb^*\|_\Hb^2} {\sigma^2},\quad a = 
\kappa(\Nr) R \sqrt{N}.
\] 
\end{theorem}
Note that the result in Theorem \ref{thm:SGD<ridge} holds for arbitrary $\lambda$. Then this theorem provides a sufficient condition for SGD such that it provably performs no worse than optimal ridge regression solution (i.e., ridge regression with optimal $\lambda$). Besides, we would also like to point out that the SGD stepsize $\gamma^*$ in Theorem \ref{thm:SGD<ridge} is only a function of the regularization parameter $\lambda$ and $\tr(\Hb)$, which can be easily estimated from training dataset without knowing the exact formula of $\Hb$.

Different from the one-hot case, here the required sample size for SGD depends on two important quantities: $R^2$ and $\kappa(\Nr)$. In particular,  $R^2=\|\wb^*\|_\Hb^2/\sigma^2$ can be understood as the \emph{signal-to-noise} ratio.
The quantity $\kappa(\Nr)$ characterizes the flatness of the eigenspectrum of $\Hb$ in the top $\Nr$-dimensional subspace, which clearly satisfies $\kappa(\Nr)\ge 1$.  Let us further explain why we have the dependencies on $R^2$ and $\kappa(\Nr)$ in the condition of the sample inflation for SGD.

A large $R^2$ emphasizes the problem hardness is more from the numerical optimization instead of from the statistic learning.
In particular, let us consider a special case where $\sigma=0$ and $R^2=\infty$, i.e., there is no noise in the least square problem, and thus solving it is purely a numerical optimization issue.
In this case, ridge regression with $\lambda=0$ achieves \emph{zero} population risk so long as the observed data can span the whole parameter space, but constant stepsize SGD in general suffers a non-zero risk in finite steps, thus cannot be competitive with the risk of ridge regression, which is as predicted by Theorem \ref{thm:SGD<ridge}.
From a learning perspective, a constant or even small $R^2$ is more interesting.


To explain why the dependency on $\kappa(\Nr)$ is unavoidable, we can consider a $2$-d dimensional example where 
\begin{align*}
\Hb =
\begin{pmatrix}
1 & 0  \\
0 & \frac{1}{\Nr\cdot \kappa(\Nr)}
\end{pmatrix}
,\quad
\wb^* =
\begin{pmatrix}
0   \\
\Nr \cdot\kappa(\Nr)
\end{pmatrix}.
\end{align*}
It is commonly known that for this problem, ridge regression with $\lambda=0$ can achieve $ \cO(\sigma^2/\Nr)$ excess risk bound \citep{friedman2001elements}. However, this problem is rather difficult for SGD since it is hard to learn the second coordinate of $\wb^*$ using gradient information (the gradient in the second coordinate is quite small). In fact, in order to accurately learn $\wb^*[2]$, SGD requires at least $\Omega(1/\lambda_2) = \Omega\big(\Nr\kappa(\Nr)\big)$ iterations/samples, which is consistent with our theory.

Then from Theorem \ref{thm:SGD<ridge} it can be observed that when the signal-to-noise ratio is nearly a constant, i.e., $R^2=\Theta(1)$, and the eigenspectrum of $\Hb$ does not decay too fast so that  $\kappa(\Nr)\le \mathrm{polylog}(\Nr)$, SGD provably generalizes no worse than ridge regression, provided with logarithmically more samples than that provided to ridge regression. More specifically, the following corollary gives a family of problem instances that are in this regime.
\begin{coro}\label{cor:SGD<ridge}
Under the same conditions as Theorem \ref{thm:SGD<ridge}, let $\Nr$ be the sample size of ridge regression. Consider the problem instance that satisfies $R^2=\Theta(1)$, $d=O(\Nr)$,   and $\lambda_i = 1/i^\alpha$ for some $\alpha\le 1$, then SGD, with a tuned stepsize $\gamma^*$,  provably generalizes no worse than any ridge regression solution in the generalizable regime if 
\begin{align*}
\Ns \ge \log^2(\Nr)\cdot\Nr.
\end{align*}
\end{coro}

We would like to further point out that the comparison made in Corollary \ref{cor:SGD<ridge} concerns the worst-case result regarding $\wb^*$ (from the perspective of SGD), while SGD could perform much better if $\wb^*$ has a nice structure. For example, considering the same setting in Corollary \ref{cor:SGD<ridge} but assuming that the ground truth $\wb^*$ is drawn from a prior distribution that is rotation invariant, SGD can be no worse than ridge regression provided the same or larger sample size. We formally state this result in the following corollary. 

\begin{coro}\label{cor:SGD<ridge_random}
Under the same conditions as Corollary \ref{cor:SGD<ridge}, let $\Nr$ be the sample size of ridge regression. Consider the problem instance with random and rotation invariant $\wb^*$, then SGD with a tuned stepsize $\gamma^*$ provably generalizes no worse than any ridge regression solution in the generalizable regime if 
\begin{align*}
\Ns \ge \Nr.
\end{align*}
\end{coro}


The next theorem shows that, in fact, for some instances, SGD could perform much better than ridge regression, as for the one-hot least square problems.


\begin{theorem}[Best-case comparison, Gaussian data]\label{thm:bestcase_gaussian}
There exists a Gaussian least square problem instance satisfying 
$R^2=1$ 
and $\kappa(N_{\mathrm{sgd}})=\Theta(1)$,
and an SGD solution with a constant stepsize and sample size $N_{\mathrm{sgd}}$, such that for any ridge regression solution (i.e., any $\lambda$) with sample size
\begin{align*}
\Nr \le \frac{\Ns^2}{\log^2(\Ns)},
\end{align*}
it holds that,
\begin{align*}
 L\big[\wb_{\mathrm{ridge}}(N_{\mathrm{ridge}};\lambda)\big]-L(\wb^*) \gtrsim L\big[\wb_{\mathrm{sgd}}(N_{\mathrm{sgd}}; \gamma^*)\big]-L(\wb^*).
\end{align*}
\end{theorem}

Besides the instance-wise comparison, it is also interesting to see under what condition SGD can provably outperform ridge regression, i.e., achieving comparable or smaller excess risk using the \emph{same} number of samples. The following theorem shows that this occurs when the signal-to-noise ratio $R^2$ is a constant and there is only a small fraction of $\wb^*$ living in the tail eigenspace of $\Hb$. 
\begin{theorem}[SGD outperforms ridge regression, Gaussian data]\label{thm:general_good_case}
Let $N_{\mathrm{ridge}}$ be sample size of ridge regression and $k^* = \min\big\{k:\lambda_k\le\frac{\tr(\Hb)}{\Nr\log(\Nr)}\big\}$, then if $R^2=\Theta(1)$, and
\begin{align*}
\sum_{i=k^*+1}^{\Nr}\lambda_i(\wb^*[i])^2 \lesssim\frac{k^*\|\wb^*\|_{\Hb}^2}{\Nr},
\end{align*}
for any ridge regression solution that is generalizable and any $\lambda$, there exists a choice of stepsize $\gamma^*$ for SGD such that
\begin{align*}
L\big[\wb_{\mathrm{sgd}}(N_{\mathrm{sgd}}; \gamma^*)\big]-L(\wb^*)\lesssim L\big[\wb_{\mathrm{ridge}}(N_{\mathrm{ridge}};\lambda)\big]-L(\wb^*)
\end{align*}
provided the sample size of SGD satisfies
\begin{align*}
\Ns\ge\Nr.
\end{align*}
\end{theorem}

\paragraph{Experiments.}
We perform experiments on Gaussian least square problem. We consider $6$ problem instances, which are the combinations of $2$ different covariance matrices $\Hb$: $\lambda_i=i^{-1}$ and $\lambda_i=i^{-2}$; and $3$ different true model parameter vectors $\wb^*$: $\wb^*[i]=1$, $\wb^*[i]=i^{-1}$, and $\wb^*[i] = i^{-10}$. Figure \ref{fig1} compares the required sample sizes of ridge regression and SGD that lead to the same population risk  on these $6$ problem instances, where the hyperparameters (i.e., $\gamma$ and $\lambda$) are fine-tuned to achieve the best performance. We have two key observations: (1) in terms of the worst problem instance for SGD (i.e., $\wb^*[i]=1$), its sample size is only  worse than ridge regression up to nearly constant factors (the curve is nearly linear); and (2) SGD can significantly outperform ridge regression when the true model $\wb^*$ mainly lives in the head eigenspace of $\Hb$ (i.e., $\wb^*[i]=i^{-10}$). The empirical observations are pretty consistent with our theoretical findings and again demonstrate the benefit of the implicit regularization of SGD. 

\begin{figure}[!t]
\vskip -0.1in
     \centering
     \subfigure[$\lambda_i=i^{-1}$]{\includegraphics[width=0.45\textwidth]{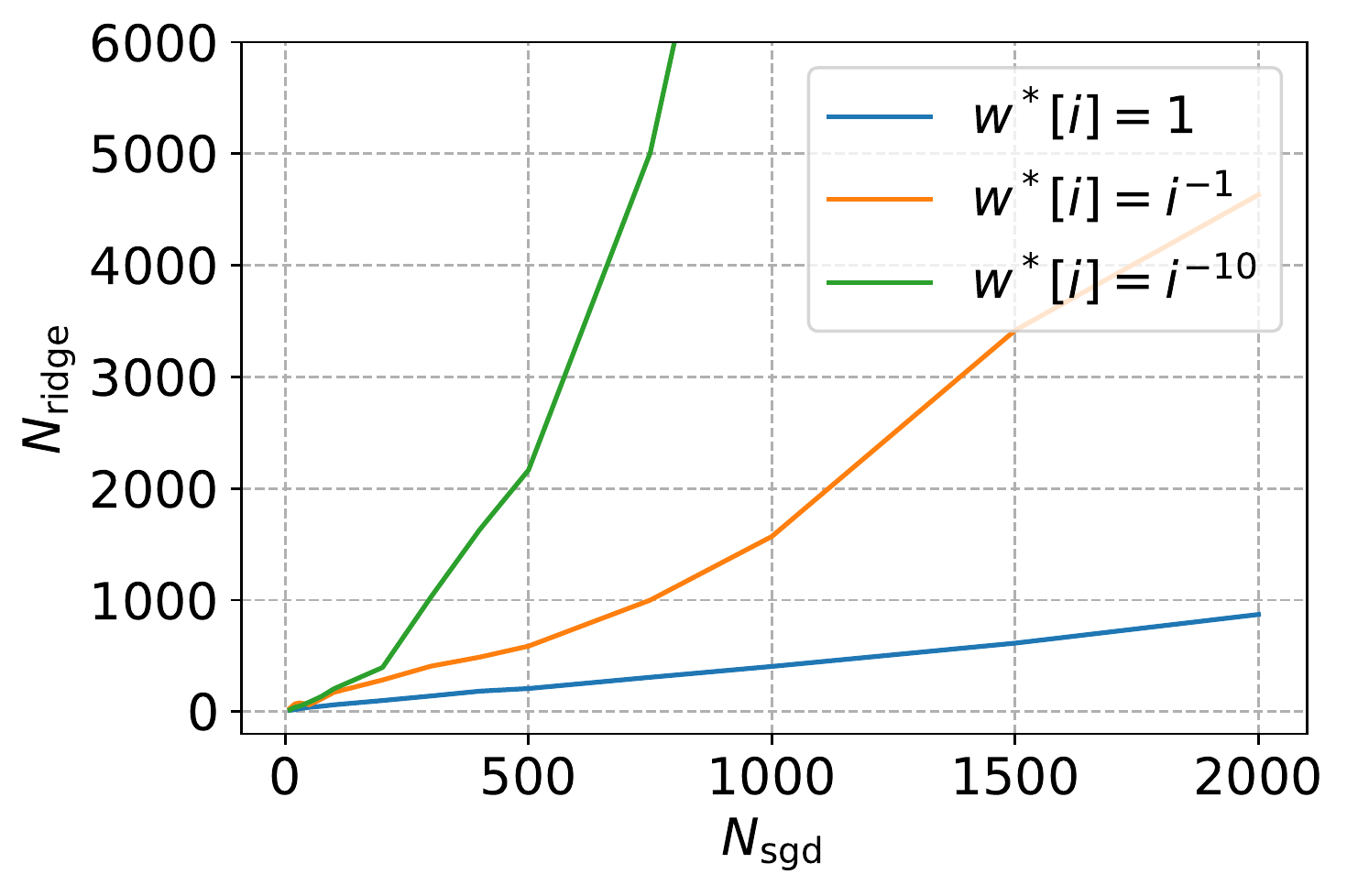}}
      \subfigure[$\lambda_i=i^{-2}$]{\includegraphics[width=0.45\textwidth]{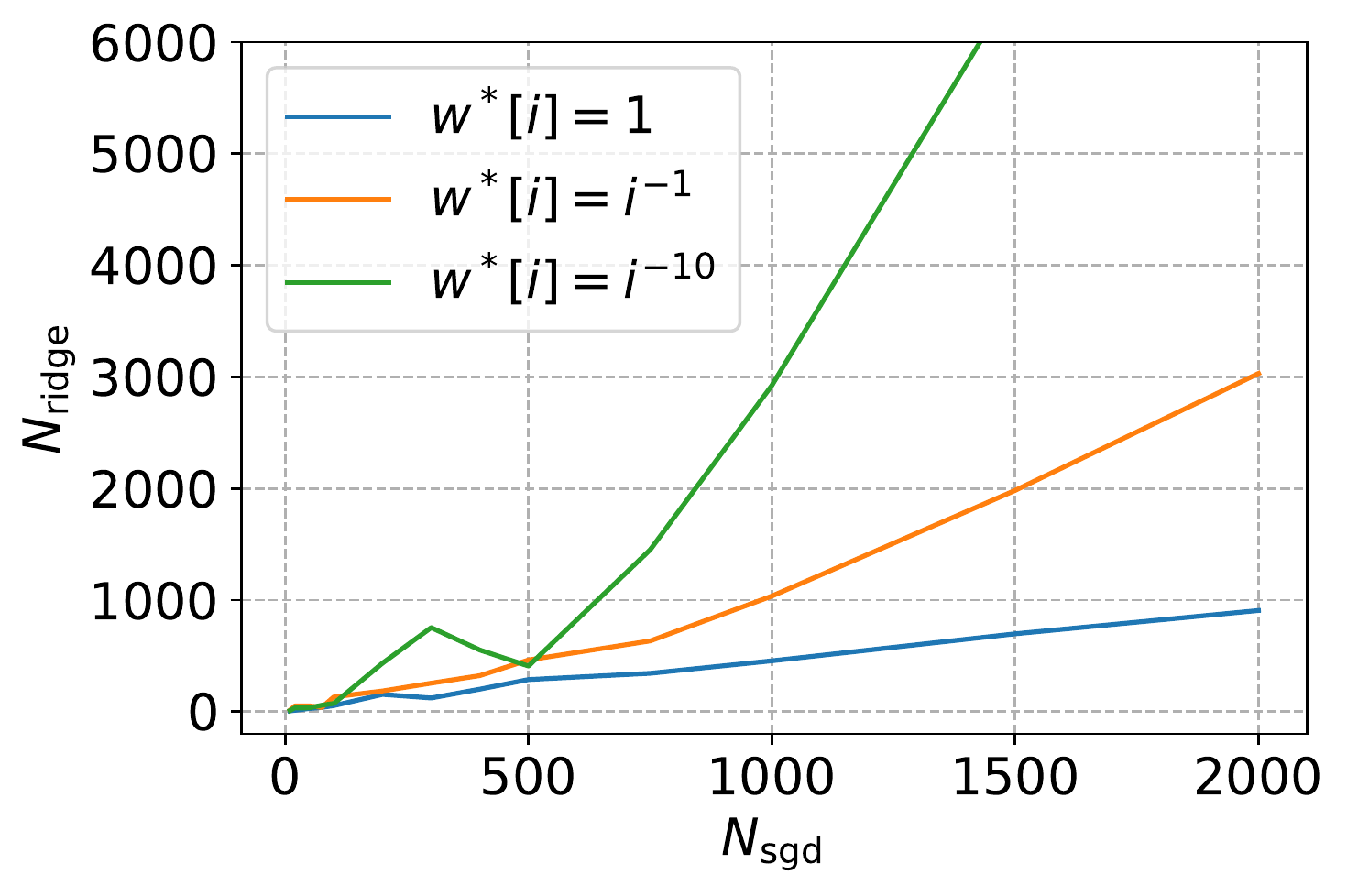}}
    
    \caption{Sample size comparison between SGD and ridge regression, where the stepsize $\gamma$ and regularization parameter $\lambda$ are fine-tuned to achieve the best performance. The problem dimension is $d=200$ and the variance of model noise is $\sigma^2=1$. We consider $6$ combinations of $2$ different covariance matrices and $3$ different ground truth model vectors. The plots are averaged over $20$ independent runs.}
    \label{fig1}
  
\end{figure}
\section{An Overview of the Proof}
In this section, we will sketch the proof of main Theorems for Gaussian least squares problems.  Recall that we aim to show that provided certain number of training samples, SGD is guaranteed to generalize better than ridge regression. Therefore, we will compare the risk \textit{upper bound} of SGD \citep{zou2021benign}  with the risk \textit{lower bound} of ridge regression \citep{tsigler2020benign}\footnote{The lower bound of ridge regression in our paper is a tighter variant of the lower bound in \citet{tsigler2020benign} since we consider Gaussian case and focus on the expected excess risk. \citet{tsigler2020benign} studied the sub-Gaussian case and established a high-probability risk bound.}. In particular, we first provide the following informal lemma summarizing the aforementioned risk bounds of SGD and ridge regression.

\begin{lemma}[Risk bounds of SGD and ridge regression, informal]\label{lemma:riskbound_informal}
Suppose Assumptions \ref{assump:model_noise} and \ref{assump:data_distribution} hold and $\gamma\le1/\tr(\Hb)$, then SGD has the following risk upper bound for arbitrary $k_1,k_2\in[d]$,
\begin{align}\label{eq:sgdupperbound}
\mathrm{SGDRisk}
&\lesssim \underbrace{\frac{1}{\gamma^2\Ns^2}\cdot\big\|\exp(-\Ns\gamma\Hb )\wb^*\big\|_{\Hb_{0:k_1}^{-1}}^2 + \|\wb^*\big\|_{\Hb_{k_1:\infty}}^2}_{\mathrm{SGDBiasBound}} \notag\\
&\qquad + \underbrace{(1+R^2)\sigma^2\cdot\bigg(\frac{k_2}{\Ns}+\Ns\gamma^2\sum_{i>k_2}\lambda_i^2\bigg)}_{\mathrm{SGDVarianceBound}}.
\end{align}
Additionally, ridge regression has the following risk lower bound for a constant $\tilde\lambda$, depending on $\lambda$, $\Nr$, and $\Hb$, and $k^*=\min\{k:\Nr\lambda_k\lesssim\tilde\lambda\}$
\begin{align}\label{eq:ridgelowerbound}
\mathrm{RidgeRisk}&\gtrsim \underbrace{\rbr{\frac{\tilde{\lambda}}{\Nr}}^2 \nbr{\wb^*}^2_{\Hb^{-1}_{0:k^*}} + \nbr{\wb^*}^2_{\Hb_{k^*:\infty}}}_{\mathrm{RidgeBiasBound}}
    + \underbrace{\sigma^2\cdot \rbr{\frac{k^*}{\Nr} + \frac{\Nr}{\tilde{\lambda}^2} \sum_{i>k^*}\lambda_i^2 }}_{\mathrm{RidgeVarianceBound}}.
\end{align}

\end{lemma}

We first highlight some useful observations in Lemma \ref{lemma:riskbound_informal}.
\begin{enumerate}[leftmargin=*]
    \item SGD has a condition on the stepsize: $\gamma\le1/\tr(\Hb)$, while ridge regression has no condition on the regularization parameter $\lambda$. \label{obs1}
    \item Both the upper bound of SGD and the lower bound of ridge regression can be decomposed into two parts corresponding to the head and tail eigenspaces of $\Hb$. Furthermore, for the upper bound of SGD, the decomposition is arbitrary ($k_1$ and $k_2$ are arbitrary), while for the lower bound of ridge estimator, the decomposition is fixed (i.e., $k^*$ is fixed). \label{obs2}
    \item Regarding the $\mathrm{SGDBiasBound}$ and $\mathrm{SGDVarianceBound}$, performing the transformation $N\rightarrow \alpha N$ and $\gamma\rightarrow\alpha^{-1}\gamma$ will decrease $\mathrm{SGDVarianceBound}$ by a factor of $\alpha$ while the $\mathrm{SGDBiasBound}$ remains unchanged. \label{obs3}
\end{enumerate}
Based on the above useful observations, we can now interpret the proof sketch for Theorems \ref{thm:SGD<ridge}, \ref{thm:bestcase_gaussian}, and \ref{thm:general_good_case}. We will first give the sketch for Theorem \ref{thm:general_good_case} and then prove Theorem \ref{thm:bestcase_gaussian} for the ease of presentation.
We would like to emphasize that the calculation in the proof sketch may not be the sharpest since they are presented for the ease of exposition. A preciser and sharper calculation can be found in Appendix.   

\paragraph{Proof Sketch of Theorem \ref{thm:SGD<ridge}.}
In order to perform instance-wise comparison, we need to take care of all possible $\wb^*\in\cH$. Therefore, by Observation \ref{obs2}, we can simply pick $k_1=k_2=k^*$ in the upper bound \eqref{eq:sgdupperbound}. Then it is clear that if setting $\gamma = \tilde\lambda^{-1}$ and $\Ns = \Nr$, we have
\begin{align*}
\mathrm{SGDBiasBound} &\le \mathrm{RidgeBiasBound}\notag\\
\mathrm{SGDVarianceBound} &= (1+R^2)\cdot\mathrm{RidgeVarianceBound}.
\end{align*}
Then by Observation \ref{obs3}, enlarging $\Ns$ by $(1+R^2)$ times suffices to guarantee
\begin{align*}
\mathrm{SGDBiasBound}+\mathrm{SGDVarianceBound} &\le \mathrm{RidgeBiasBound}+\mathrm{RidgeVarianceBound}.
\end{align*}
On the other hand, according to Observation \ref{obs1}, there is an upper bound on the feasible stepsize of SGD:  $\gamma\le 1/\tr(\Hb)$. Therefore, the above claim only holds when $\tilde \lambda \ge \tr(\Hb)$. 

When $\tilde\lambda\le \tr(\Hb)$, the stepsize $\tilde\lambda^{-1}$ is no longer feasible and instead, we will use the largest possible stepsize: $\gamma = 1/\tr(\Hb)$. Besides, note that we assume ridge regression solution is in the generalizable regime, then it holds that $k^*\le \Nr$ since otherwise we have
\begin{align*}
\mathrm{RidgeRisk}\gtrsim \mathrm{RidgeVarianceBound}\ge \sigma^2.
\end{align*}
Then again we set $k_1=k_2=k^*$ in $\mathrm{SGDBiasBound}$ and $\mathrm{SGDVarianceBound}$. Applying the choice of stepsize $\gamma = 1/\tr(\Hb)$ and sample size 
\begin{align*}
\Ns = \frac{\log(R^2\Nr)}{\gamma\lambda_{k^*}} \le \Nr\cdot \kappa(\Nr)\cdot \log(R^2\Nr),
\end{align*}
we get 
\begin{align}\label{eq:upperbound_SGDbiascase2}
\mathrm{SGDBiasBound}  &\le \frac{(1-\Ns\gamma\lambda_{k^*})^{\Ns}}{\gamma^2\Ns^2\lambda_{k^*}^2}\cdot \|\wb^*\|_{\Hb_{0:k^*}}^2 + \|\wb^*\|_{\Hb_{k^*:\infty}}^2\notag\\
&\le \frac{\sigma^2}{\Nr} + \|\wb^*\|_{\Hb_{k^*:\infty}}^2\notag\\
&\le \mathrm{RidgeBiasBound}+\mathrm{RidgeVarianceBound}.
\end{align}
Moreover, we can also get the following bound on $\mathrm{SGDVarianceBound}$,
\begin{align*}
\mathrm{SGDVarianceBound}
&\le (1+R^2)\sigma^2\cdot\bigg(\frac{k^*}{\Nr} + \frac{\log(R^2\Nr)}{\lambda_{k^*}\tr(\Hb)}\sum_{i> k^*}\lambda_i^2\bigg)\notag\\
&\le (1+R^2)\log(R^2\Nr)\cdot\mathrm{RidgeVarianceBound},
\end{align*}
where in the second inequality we use the fact that \begin{align*}
\frac{\Nr}{\tilde\lambda^2}\ge \frac{1}{\lambda_{k^*}\tilde\lambda}\ge \frac{1}{\lambda_{k^*}\tr(\Hb)}.
\end{align*}
Therefore by Observation \ref{obs3} again we can enlarge $\Ns$ properly to ensure that $\mathrm{SGDVarianceBound}$ remains unchanged and $\mathrm{SGDVarianceBound}\le \mathrm{RidgeVarianceBound}$. Then combining this and \eqref{eq:upperbound_SGDbiascase2} we can get
\begin{align*}
\mathrm{SGDBiasBound} + \mathrm{SGDVarianceBound}\le 2\cdot\mathrm{RidgeBiasBound} + 2\cdot\mathrm{RidgeVarianceBound},
\end{align*}
which completes the proof.

\paragraph{Proof Sketch of Theorem \ref{thm:general_good_case}.} 
Now we will investigate in which regime SGD will generalizes no worse than ridge regression when provided with same training sample size. For simplicity in the proof we assume $R^2=1$. First note that we only need to deal with the case where $\tilde\lambda\le \tr(\Hb)$ by the proof sketch of Theorem \ref{thm:SGD<ridge}.

Unlike the instance-wise comparison that consider all possible $\wb^*\in\cH$, in this lemma we only consider the set of $\wb^*$ that SGD performs well. Specifically, as we have shown in the proof of Theorem \ref{thm:SGD<ridge}, in the worst-case comparison (in terms of $\wb^*$), we require SGD to be able to learn the first $k^*$ (where $k^*\le \Nr$) coordinates of $\wb^*$ in order to be competitive with ridge regression, while SGD with sample size $\Ns$ can only be guaranteed to learn the first $\ks$ coordinates of $\wb^*$, where $\ks = \min\{k:\Nr\lambda_k\le \tr(\Hb)\}$. Therefore, in the instance-wise comparison we need to enlarge $\Ns$ to $\Nr\cdot\kappa(\Nr)$ to guarantee the learning of the top $k^*$ coordinates of $\wb^*$.  

However, this is not required for some good $\wb^*$'s that have small components in the $\ks$-$k^*$ 
coordinates. In particular, as assumed in the theorem, we have $\sum_{i=\hat k+1}^{\Nr}\lambda_i(\wb^*[i])^2\le \hat k\|\wb^*\|_\Hb^2/\Nr$, where $\hat k:=\min\{k:\lambda_k\Ns\le \tr(\Hb)\cdot\log(\Ns)\}$ satisfies $\hat k \le \ks\le k^*$.
Then let $k_1=\hat k$ in $\mathrm{SGDBiasBound}$,  we have
\begin{align*}
\mathrm{SGDBiasBound} &= \frac{1}{\gamma^2\Nr^2}\cdot\big\|\exp(-\Nr\gamma\Hb)\wb^*\big\|_{\Hb_{0:\hat k}^{-1}}^2  + \|\wb^*\big\|_{\Hb_{\hat k:\infty}}^2\notag\\
&\le (1-\Nr\gamma\lambda_{\hat k})^{\Nr}\cdot \|\wb^*\|_{\Hb_{0:k^*}}^2 + \|\wb^*\|_{\Hb_{\hat k:\infty}}^2\notag\\
&\overset{(i)}\le \frac{R^2\sigma^2(\hat k+1)}{\Nr} + \|\wb^*\|_{\Hb_{k^*:\infty}}^2\notag\\
&\le 2\cdot\mathrm{RidgeVarBound} + \mathrm{RidgeBiasBound}.
\end{align*}
where $(i)$ is due to the condition that $\sum_{i=\hat k+1}^{\Nr}\lambda_i(\wb^*[i])^2\le \hat k\|\wb^*\|_\Hb^2/\Nr$.
Moreover, it is easy to see that given $\Ns = \Nr$ and $\gamma=1/\tr(\Hb)\le 1/\tilde\lambda$, we have $\mathrm{SGDVarianceBound}\le 2\cdot\mathrm{RidgeVarianceBound}$. As a consequence we can get 
\begin{align*}
\mathrm{SGDBiasBound} + \mathrm{SGDVarianceBound}\le 3\cdot\mathrm{RidgeBiasBound} + 3\cdot\mathrm{RidgeVarianceBound}.
\end{align*}

\paragraph{Proof Sketch of Theorem \ref{thm:bestcase_gaussian}.}
We will consider the best $\wb^*$ for SGD, which only has nonzero entry in the first coordinate. For example, consider a true model parameter vector with $\wb^*[1] = 1$ and $\wb^*[i]=0$ for $i\ge 2$ and a problem instance whose spectrum of $\Hb$ has a flat tail with $\sum_{i\ge \Nr}\lambda_i^2=\Theta(1)$ and $\sum_{i\ge 2}\lambda_i^2=\Theta(1)$. Then according to Lemma \ref{lemma:riskbound_informal}, we can set the stepsize as $\gamma = \Theta(\log(\Ns)/\Ns)$ and get
\begin{align*}
\mathrm{SGDRisk} &\lesssim \mathrm{SGDBiasBound} + \mathrm{SGDVarianceBound}\notag\\
&= O\bigg(\frac{1}{\Ns} + \frac{\log^2(\Ns)}{\Ns}\bigg) = O\bigg(\frac{\log^2(\Ns)}{\Ns}\bigg).
\end{align*}

For ridge regression, according to Lemma \ref{lemma:riskbound_informal}  we have
\begin{align*}
\mathrm{RidgeRisk} &\gtrsim  \mathrm{RidgeBiasBound} + \mathrm{RidgeVarianceBound}\notag\\
&=\Omega\bigg(\frac{\tilde \lambda^2}{\Nr^2} +  \frac{\Nr}{\tilde\lambda^2}\bigg)\qquad \text{since } \sum_{i\ge k^*}\lambda_i^2=\Theta(1) \notag\\
& = \Omega\bigg(\frac{1}{\Nr^{1/2}}\bigg). \qquad \text{by the fact that } a+b\ge\sqrt{ab}
\end{align*}
Therefore, it is evident that ridge regression is guaranteed to be worse than SGD if $\Nr\le \Ns^2/\log^2(\Ns)$. This completes the proof.

\section{Conclusions}
We conduct an instance-based risk comparison between SGD and ridge regression for a broad class of least square problems. We show that SGD is always no worse than ridge regression provided logarithmically more samples. On the other hand, there exist some instances where even optimally-tuned ridge regression needs quadratically more samples to compete with SGD. This separation in terms of sample inflation between SGD and ridge regression suggests a provable benefit of implicit regularization over explicit regularization for least squares problems. In the future, we will explore the benefits of implicit regularization for learning other linear models and potentially nonlinear models.

\section*{Acknowledgments and Disclose of Funding}

We would like to thank the anonymous reviewers and area chairs for their helpful comments. 
DZ is supported by the Bloomberg Data Science Ph.D. Fellowship. JW is supported in part
by NSF CAREER grant 1652257. VB is supported in part by NSF CAREER grant 1652257,
ONR Award N00014-18-1-2364 and the Lifelong Learning Machines program from
DARPA/MTO.
QG is supported in part by the National Science Foundation awards IIS-1855099 and IIS-2008981.
SK acknowledges funding from the National Science Foundation under Award CCF-1703574. The views and conclusions contained in this paper are those of the authors and should not be interpreted as representing any funding agencies.

\bibliographystyle{plainnat}
\bibliography{refs}

\newpage
\section*{Checklist}


\begin{enumerate}

\item For all authors...
\begin{enumerate}
  \item Do the main claims made in the abstract and introduction accurately reflect the paper's contributions and scope?
    \answerYes{}
  \item Did you describe the limitations of your work?
    \answerYes{}
  \item Did you discuss any potential negative societal impacts of your work?
    \answerNA{This paper focuses on theoretical explanations of the implicit regularization of SGD and its comparison to explicit regularization in ridge regression. It has no potential negative societal impact.}
  \item Have you read the ethics review guidelines and ensured that your paper conforms to them?
    \answerYes{}
\end{enumerate}

\item If you are including theoretical results...
\begin{enumerate}
  \item Did you state the full set of assumptions of all theoretical results?
    \answerYes{}
	\item Did you include complete proofs of all theoretical results?
    \answerYes{}
\end{enumerate}

\item If you ran experiments...
\begin{enumerate}
  \item Did you include the code, data, and instructions needed to reproduce the main experimental results (either in the supplemental material or as a URL)?
    \answerNA
  \item Did you specify all the training details (e.g., data splits, hyperparameters, how they were chosen)?
    \answerNA
	\item Did you report error bars (e.g., with respect to the random seed after running experiments multiple times)?
    \answerNA
	\item Did you include the total amount of compute and the type of resources used (e.g., type of GPUs, internal cluster, or cloud provider)?
    \answerNA
\end{enumerate}

\item If you are using existing assets (e.g., code, data, models) or curating/releasing new assets...
\begin{enumerate}
  \item If your work uses existing assets, did you cite the creators?
    \answerNA
  \item Did you mention the license of the assets?
    \answerNA
  \item Did you include any new assets either in the supplemental material or as a URL?
    \answerNA
  \item Did you discuss whether and how consent was obtained from people whose data you're using/curating?
   \answerNA
  \item Did you discuss whether the data you are using/curating contains personally identifiable information or offensive content?
    \answerNA
\end{enumerate}

\item If you used crowdsourcing or conducted research with human subjects...
\begin{enumerate}
  \item Did you include the full text of instructions given to participants and screenshots, if applicable?
    \answerNA
  \item Did you describe any potential participant risks, with links to Institutional Review Board (IRB) approvals, if applicable?
    \answerNA
  \item Did you include the estimated hourly wage paid to participants and the total amount spent on participant compensation?
    \answerNA
\end{enumerate}

\end{enumerate}

 \newpage
 \appendix


\section{Proof of One-hot Least Squares}
\subsection{Excess risk bound of SGD}
In this part we will mainly follow the proof technique in \citet{zou2021benign} that is developed to sharply characterize the excess risk bound for SGD (with tail-averaging) when the data distribution has a nice finite fourth-moment bound. However, such condition does not hold for the one-hot case so that their results cannot be directly applied here. 

Before presenting the detailed proofs, we first introduce some notations and definitions that will be repeatedly used in the subsequent analysis. 
Let $\Hb=\EE[\xb\xb^\top]$ be the covariance of data distribution. It is easy to verify that $\Hb$ is a diagonal matrix with eigenvalues $\lambda_1,\dots,\lambda_d$.
Let $\wb_t$ be the $t$-th iterate of the SGD, we define $\betab_t:=\wb_t-\wb^*$ as the centered SGD iterate. Then we define $\betab_t^{\bias}$ and $\betab_t^{\var}$ as the bias error and variance error respectively, which are described by the following update rule:
\begin{align}\label{eq:update_rule_eta_t}
\betab_t^{\bias} &= \big(\Ib-\gamma\xb_t\xb_t^\top\big)\betab_{t-1}^{\bias},\qquad \betab_0^{\bias}=\betab_0,\notag\\
\betab_t^{\var} &= \big(\Ib-\gamma\xb_t\xb_t^\top\big)\betab_{t-1}^{\bias} + \gamma\xi_t\xb_t,\qquad\betab_0^{\var}=\boldsymbol{0}.
\end{align}
Accordingly, we can further define the bias covariance $\Bb_t$ and variance covariance $\Cb_t$ as follows
\begin{align*}
\Bb_t = \EE[\betab_t^{\bias}\otimes\betab_t^{\bias}],\qquad \Cb_t = \EE[\betab_t^{\var}\otimes\betab_t^{\var}].
\end{align*}
Regarding these two covariance matrices, the following lemma mathematically characterizes the upper bounds of the diagonal entries of $\Bb_t$ and $\Cb_t$.
\begin{lemma}\label{lemma:bound_diagonal_covariance}
Under Assumptions \ref{assump:model_noise}, let $\bar\Bb_t = \diag(\Bb_t)$ and $\bar\Cb_t = \diag(\Cb_t)$, then if the stepsize satisfies $\gamma\le 1$, we have
\begin{align*}
\bar\Bb_t \preceq (\Ib-\gamma\Hb)\bar\Bb_{t-1},\qquad
\bar\Cb_t \preceq (\Ib-\gamma\Hb)\bar\Cb_{t-1}+\gamma^2\sigma^2\Hb.
\end{align*}
\end{lemma}
\begin{proof}
According to \eqref{eq:update_rule_eta_t}, we have
\begin{align}\label{eq:update_Bb_t}
\Bb_t = \EE[\betab_t^{\bias}\otimes\betab_t^{\bias}] &= \EE\big[(\Ib-\gamma\xb_t\xb_t^\top)\betab_{t-1}^{\bias}\otimes (\Ib-\gamma\xb_t\xb_t^\top)\betab_{t-1}^{\bias}\big]\notag\\
&=\Bb_{t-1} - \gamma\Hb\Bb_{t-1}-\gamma\Bb_{t-1}\Hb + \gamma^2\EE[\xb_t\xb_t^\top\Bb_{t-1}\xb_t\xb_t^\top].
\end{align}
Note that $\xb_t=\eb_i$ with probability $\lambda_i$, then we have
\begin{align*}
\EE[\xb_t\xb_t^\top\Bb_{t-1}\xb_t\xb_t^\top] &= \sum_{i}\lambda_i \cdot\eb_i\eb_i^\top\Bb_{t-1}\eb_i\eb_i^\top\notag\\
& = \sum_{i}\lambda_i\cdot \eb_i^\top\Bb_{t-1}\eb_i\cdot\eb_i\eb_i^\top\notag\\
& = \bar\Bb_{t-1}\Hb.
\end{align*}
Plugging the above equation into \eqref{eq:update_Bb_t} gives
\begin{align*}
\Bb_t = \Bb_{t-1} - \gamma\Hb\Bb_{t-1}-\gamma\Bb_{t-1}\Hb + \gamma^2\bar\Bb_{t-1}\Hb.
\end{align*}
Then if only look at the diagonal entries of both sides, we have
\begin{align*}
\bar\Bb_t = \bar\Bb_{t-1} - 2\gamma\Hb\bar\Bb_{t-1} + \gamma^2\Hb\bar\Bb_{t-1} \preceq (\Ib-\gamma\Hb)\bar\Bb_{t-1},
\end{align*}
where in the first equation we use the fact that $\diag(\Hb\Bb) = \diag(\Bb\Hb)=\Hb\bar\Bb$ and the inequality follows from the fact that both $\bar\Bb_t$ and $\Hb$ are diagonal and $\gamma\le1$.

Similarly, regarding $\Cb_t$ the following holds according to \eqref{eq:update_rule_eta_t},
\begin{align*}
\Cb_t = \EE\big[(\Ib-\gamma\xb_t\xb_t^\top)\betab_{t-1}^{\var}\otimes(\Ib-\gamma\xb_t\xb_t^\top)\betab_{t-1}^{\var}\big] +\gamma^2\EE[\xi_t^2\xb_t\xb_t^\top],
\end{align*}
where we use the fact that $\EE[\xi_t|\xb_t]=0$. Similar to deriving the bound for $\bar\Bb_t$, we have
\begin{align*}
\diag\big(\EE\big[(\Ib-\gamma\xb_t\xb_t^\top)\betab_t^{\var}\otimes(\Ib-\gamma\xb_t\xb_t^\top)\betab_t^{\var}\big]\big)\preceq (\Ib-\gamma\Hb)\bar\Cb_{t-1}.
\end{align*}
Besides, under Assumption \ref{assump:model_noise} we also have $\EE[\xi_t^2\xb_t\xb_t^\top]=\sigma^2\Hb$, which is a diagonal matrix. Based on these two results, we can get the following upper bound for $\bar\Cb_t$,
\begin{align*}
\bar\Cb_t &= \diag\big(\EE\big[(\Ib-\gamma\xb_t\xb_t^\top)\betab_{t-1}^{\var}\otimes(\Ib-\gamma\xb_t\xb_t^\top)\betab_{t-1}^{\var}\big] +\gamma^2\EE[\xi_t^2\xb_t\xb_t^\top]\big)\notag\\
&\preceq (\Ib-\gamma\Hb)\bar\Cb_{t-1}+\gamma^2\sigma^2\Hb.
\end{align*}
This completes the proof.
\end{proof}

\begin{lemma}[Lemmas D.1 \& D.2 in \citet{zou2021benign}]\label{lemma:bias_var_decomposition_sgd_onehot}
Let $\bar\wb_{N:2N}$ be the output of tail-averaged SGD, then if the stepsize satisfied $\gamma\le 1/\lambda_1$, it holds that
\begin{align*}
\EE[L(\bar\wb_{N:2N})]-L(\wb^*) \lesssim \sgdbias + \sgdvar,
\end{align*}
where 
\begin{align*}
\sgdbias&\le \frac{1}{N^2}\sum_{t=0}^{N-1}\sum_{k=t}^{N-1}\big\la(\Ib-\gamma\Hb)^{k-t}\Hb,\Bb_{N+t}\big\ra\notag\\
\sgdvar&\le \frac{1}{N^2}\sum_{t=0}^{N-1}\sum_{k=t}^{N-1}\big\la(\Ib-\gamma\Hb)^{k-t}\Hb,\Cb_{N+t}\big\ra
\end{align*}
\end{lemma}

\begin{lemma}\label{lemma:excessrisk_sgd_onehot}
Under Assumptions \ref{assump:model_noise}, if the stepsize satisfies $\gamma\le 1$ and set $\wb_0=\boldsymbol{0}$, then 
\begin{align*}
\EE[L(\bar\wb_{N:2N})]-L(\wb^*) \le 2\cdot\bias + 2\cdot\var,
\end{align*}
where 
\begin{align*}
\bias &\lesssim \frac{1}{N^2\gamma^2}\cdot \big\|(\Ib-\gamma\Hb)^{N/2}\wb^*\big\|_{\Hb_{0:k_1}^{-1}} + \big\|(\Ib-\gamma\Hb)^{N/2}\wb^*\big\|_{\Hb_{k_1:\infty}}^2\notag\\
\var &\lesssim\sigma^2\cdot\bigg(\frac{k_2}{N} + N\gamma^2\sum_{i>k_2}\lambda_i^2\bigg)
\end{align*}
for arbitrary $k_1,k_2\in[d]$.
\end{lemma}
\begin{proof}
The first conclusion of this theorem can be directly proved via Young's inequality.

Note that $\Hb$ is a diagonal matrix, and thus $(\Ib-\gamma\Hb)^{k-t}$ is also a diagonal matrix for all $k$ and $t$. Therefore, by Lemma \ref{lemma:bias_var_decomposition_sgd_onehot}, it is clear that in order to calculate the upper bound of the bias and variance error, it suffices to consider the diagonal entries of $\Bb_{N+t}$ and $\Cb_{N+t}$, denoted by $\bar\Bb_{N+t}$ and $\bar\Cb_{N+t}$ (which are obtained by setting all non-diagonal entries of $\Bb_{N+t}$ and $\Cb_{N+t}$ as zero). Then by Young's inequality, Lemma \ref{lemma:bias_var_decomposition_sgd_onehot} implies that
\begin{align}\label{eq:bound_bias_var_step0}
\bias&\le \frac{1}{N^2}\sum_{t=0}^{N-1}\sum_{k=t}^{N-1}\big\la(\Ib-\gamma\Hb)^{k-t}\Hb,\bar\Bb_{N+t}\big\ra\notag\\
\var&\le \frac{1}{N^2}\sum_{t=0}^{N-1}\sum_{k=t}^{N-1}\big\la(\Ib-\gamma\Hb)^{k-t}\Hb,\bar\Cb_{N+t}\big\ra.
\end{align}
Now we are ready to precisely calculate the above two bounds. In particular, by Lemma \ref{lemma:bound_diagonal_covariance} we have 
\begin{align}
\bar\Bb_t &\preceq (\Ib-\gamma\Hb)\bar\Bb_{t-1}\preceq (\Ib-\gamma\Hb)^t\Bb_0,\label{eq:bound_Bt}\\
\bar\Cb_t &\preceq (\Ib-\gamma\Hb)\bar\Cb_{t-1}\preceq \sum_{s=0}^{t-1} \sigma^2\gamma^2(\Ib-\gamma\Hb)^s\Hb=\sigma^2\gamma\big(\Ib-(\Ib-\gamma\Hb)^t\big) \label{eq:bound_Ct},
\end{align}
where in the second inequality we use the fact that $\Cb_0 = \betab_{t}^{\var}\otimes\betab_t^{\var}=\boldsymbol{0}$. Then plugging \eqref{eq:bound_Bt} into \eqref{eq:bound_bias_var_step0} gives
\begin{align}\label{eq:bound_bias_onehot_step2}
\bias &\le \frac{1}{N^2}\sum_{t=0}^{N-1}\sum_{k=t}^{N-1}\la(\Ib-\gamma\Hb)^{k-t}\Hb, (\Ib-\gamma\Hb)^{N+t}\Bb_0\ra \notag\\
& = \frac{1}{N^2}\bigg\la\sum_{k=0}^{N-1-t}(\Ib-\gamma\Hb)^{k}\Hb, \sum_{t=0}^{N-1}(\Ib-\gamma\Hb)^{N+t}\Bb_0\bigg\ra\notag\\
& \le \frac{1}{N^2}\bigg\la\sum_{k=0}^{N-1}(\Ib-\gamma\Hb)^{k}\Hb, \sum_{t=0}^{N-1}(\Ib-\gamma\Hb)^{N+t}\Bb_0\bigg\ra\notag\\
& = \frac{1}{ N^2\gamma^2}\Big\la \Ib-(\Ib-\gamma\Hb)^N, \Hb^{-1}(\Ib-\gamma\Hb)^N\big(\Ib-(\Ib-\gamma\Hb)^N\big)\Bb_0\Big\ra\notag\\
& = \frac{1}{ N^2\gamma^2}\Big\la (\Ib-\gamma\Hb)^N\big[\Ib-(\Ib-\gamma\Hb)^N\big]^2\Hb^{-1}, \Bb_0\Big\ra
\end{align}
Note that $(1-x)^N\ge \min\{0, 1-Nx\}$ for all $x\in[0,1]$. Then for all $i$ we have
\begin{align*}
\big[1-(1-\gamma\lambda_i)^N]^2\lambda^{-1} \le \min\bigg\{\frac{1}{\lambda_i}, N^2\gamma^2\lambda_i\bigg\}
\end{align*}
where we use the fact that $\gamma\le 1\le 1/\lambda_i$ for all $i$. This further implies that 
\begin{align*}
\big[\Ib-(\Ib-\gamma\Hb)^N\big]^2\Hb^{-1}\preceq \Hb_{0:k}^{-1} + N^2\gamma^2\Hb_{k:\infty}
\end{align*}
for all $k\in[d]$. Plugging the above results into \eqref{eq:bound_bias_onehot_step2} leads to
\begin{align}\label{eq:bound_bias_onehot_step3}
\bias &\le \frac{1}{N^2\gamma^2}\cdot \big\la\Hb_{0:k}^{-1},(\Ib-\gamma\Hb)^N\Bb_0\big\ra + \big\la\Hb_{k:\infty},(\Ib-\gamma\Hb)^N\Bb_0\big\ra
\end{align}
for all $k\in[d]$. Further note that $\Bb_0=(\wb_0-\wb^*)\otimes(\wb_0-\wb^*) = \wb^*\otimes\wb^*$ as we pick $\wb_0=\boldsymbol{0}$. Thus \eqref{eq:bound_bias_onehot_step3} implies that
\begin{align*}
\bias &\le \frac{1}{N^2\gamma^2}\cdot \big\|(\Ib-\gamma\Hb)^{N/2}\wb^*\big\|_{\Hb_{0:k}^{-1}} + \big\|(\Ib-\gamma\Hb)^{N/2}\wb^*\big\|_{\Hb_{k:\infty}}^2.
\end{align*}

Then we will deal with the variance error. Plugging \eqref{eq:bound_Ct} into \eqref{eq:bound_bias_var_step0} gives
\begin{align*}
\var&\le \frac{\sigma^2\gamma}{N^2}\sum_{t=0}^{N-1}\sum_{k=t}^{N-1}\la(\Ib-\gamma\Hb)^{k-t}\Hb, \Ib-(\Ib-\gamma\Hb)^{N+t}\ra \notag\\
&\le \frac{\sigma^2\gamma}{N^2}\sum_{t=0}^{N-1}\bigg\la\sum_{k=0}^{N-1}(\Ib-\gamma\Hb)^{k}\Hb, \Ib-(\Ib-\gamma\Hb)^{N+t}\bigg\ra\notag\\
& = \frac{\sigma^2}{N^2}\sum_{t=0}^{N-1}\bigg\la\Ib-(\Ib-\gamma\Hb)^{N}, \Ib-(\Ib-\gamma\Hb)^{N+t}\bigg\ra\notag\\
&\le \frac{\sigma^2}{N}\Big\la\Ib-(\Ib-\gamma\Hb)^{2N}, \Ib-(\Ib-\gamma\Hb)^{2N}\Big\ra.
\end{align*}
We then use the inequality $(1-x)^N\ge\min\{0,1-xN\}$ again and thus the above inequality further leads to
\begin{align*}
\var&\le \frac{\sigma^2}{N}\cdot\sum_{i}\min\{1, 4N^2\gamma^2\lambda_i^2\}\notag\\
&\le \frac{4\sigma^2}{N}\cdot\bigg(k + N^2\gamma^2\sum_{i>k}\lambda_i^2\bigg)
\end{align*}
for any $k\in[d]$. 
\end{proof}

\subsection{Excess risk bound of ridge regression}

\begin{lemma}\label{lemma:bias_var_decomposition_ridge}
Let $\Xb\in\RR^{N\times d}$ be the training data matrix and $\wb_{\mathrm{ridge}}(N;\lambda)$ be the solution of ridge regression with parameter $\lambda$ and sample size $N$, then for any $\lambda>0$
\begin{align*}
\EE[L(\wb_{\mathrm{ridge}}(N;\lambda))]-L(\wb^*) = \bias + \var,
\end{align*}
where
\begin{align*}
\bias  &= \lambda^2\cdot\EE\big[\wb^{*\top}(\Xb^\top\Xb+\lambda\Ib)^{-1}\Hb(\Xb^\top\Xb+\lambda\Ib)^{-1}\wb^*\big]\notag\\
\var & = \sigma^2\cdot\EE\big[\tr\big((\Xb^\top\Xb+\lambda\Ib)^{-1}\Xb^\top\Xb(\Xb^\top\Xb+\lambda\Ib)^{-1}\Hb\big)\big],
\end{align*}
where the expectations are taken over the randomness of the training data matrix $\Xb$. 
\end{lemma}
\begin{proof}
Recall that the solution of ridge regression takes form
\begin{align*}
\wb_{\mathrm{ridge}}(N;\lambda) = (\Xb^\top\Xb+\lambda\Ib)^{-1}\Xb^\top\yb,
\end{align*}
where $\Xb$ is the data matrix and $\yb$ is the response vector. Then according to the definition of the loss function $L(\wb)$, we have
\begin{align*}
\EE[L(\wb_{\mathrm{ridge}}(N;\lambda))] &= \EE\Big[\big(y - \la\wb_{\mathrm{ridge}}(N;\lambda), \xb\ra\big)^2\Big] \notag\\
&= \EE\Big[\big(\la\wb^*, \xb\ra - \la\wb_{\mathrm{ridge}}(N;\lambda), \xb\ra\big)^2\Big] +  \EE\Big[\big(y - \la\wb^*, \xb\ra\big)^2\Big] \notag\\
&\qquad+ 2\EE\big[\big(\la\wb^*, \xb\ra - \la\wb_{\mathrm{ridge}}(N;\lambda), \xb\ra\big)\cdot\big(y - \la\wb^*, \xb\ra\big)\big]\notag\\
& = \EE[\|\wb_{\mathrm{ridge}}(N;\lambda)-\wb^*\|_\Hb^2] + L(\wb^*),
\end{align*}
where the last equation is by Assumption \ref{assump:model_noise}. Then regarding $\EE[\|\wb_{\mathrm{ridge}}(N;\lambda)-\wb^*\|_\Hb^2]$, let $\bxi = \yb - \Xb\wb^*$ be the model noise vector, we have
\begin{align*}
\EE[\|\wb_{\mathrm{ridge}}(N;\lambda)-\wb^*\|_\Hb^2]& = \EE\big[\big\|(\Xb^\top\Xb+\lambda\Ib)^{-1}\Xb^\top\yb-\wb^*\big\|_\Hb^2\big]\notag\\
&=\EE\big[\big\|(\Xb^\top\Xb+\lambda\Ib)^{-1}\Xb^\top(\Xb\wb^* +\bxi)-\wb^*\big\|_\Hb^2\big]\notag\\
& = \underbrace{\EE\big[\big\|(\Xb^\top\Xb+\lambda\Ib)^{-1}\Xb^\top\Xb\wb^* -\wb^*\big\|_\Hb^2\big]}_{\bias}+\underbrace{\EE\big[\big\|(\Xb^\top\Xb+\lambda\Ib)^{-1}\Xb^\top\bxi\big\|_\Hb^2\big]}_{\var}.
\end{align*}
where in the last inequality we again apply Assumption \ref{assump:model_noise} that $\EE[\bxi|\Xb] = \boldsymbol{0}$. More specifically, the bias error can be reformulated as
\begin{align*}
\bias &=\EE\big[\big\|\big((\Xb^\top\Xb+\lambda\Ib)^{-1}\Xb^\top\Xb-\Ib\big)\wb^*\big\|_\Hb^2 \big]\notag\\
&= \lambda^2\EE\big[\big\|(\Xb^\top\Xb+\lambda\Ib)^{-1}\wb\big\|_\Hb^2\big]\notag\\
& = \lambda^2\EE\big[\wb^{*\top}(\Xb^\top\Xb+\lambda\Ib)^{-1}\Hb(\Xb^\top\Xb+\lambda\Ib)^{-1}\wb^*\big].
\end{align*}
In terms of the variance error, note that by Assumption \ref{assump:model_noise} we have $\EE[\bxi\bxi^\top|\Xb]=\sigma^2\Ib$, then 
\begin{align*}
\var &= \EE\big[\big\|(\Xb^\top\Xb+\lambda\Ib)^{-1}\Xb^\top\bepsilon\big\|_\Hb^2\big]\notag\\
&= \EE\big[\tr\big((\Xb^\top\Xb+\lambda\Ib)^{-1}\Xb^\top\bxi\bxi^\top\Xb(\Xb^\top\Xb+\lambda\Ib)^{-1}\Hb\big)\big]\notag\\
& = \sigma^2\cdot\EE\big[\tr\big((\Xb^\top\Xb+\lambda\Ib)^{-1}\Xb^\top\Xb(\Xb^\top\Xb+\lambda\Ib)^{-1}\Hb\big)\big].
\end{align*}
\end{proof}

\begin{lemma}\label{lemma:excessrisk_ridge_onehot}
The solution of ridge regression with sample size $N$ and regularization parameter $\lambda$ satisfies
\begin{align*}
\EE[L(\wb_{\mathrm{ridge}}(N;\lambda))]-L(\wb^*) = \ridgebias + \ridgevar,
\end{align*}
where
\begin{align*}
\ridgebias &\gtrsim \max\bigg\{\sum_i (1-\lambda_i)^N\cdot\lambda_i\wb^*[i]^2,\sum_{i=1}^{k^*} \frac{\lambda^2\lambda_i\wb^*[i]^2}{(N\lambda_i+\lambda)^2} +\sum_{i>k^*} \lambda_i\wb^*[i]^2\bigg\} \notag\\  
\ridgevar 
&\gtrsim \sigma^2\cdot\Bigg(\sum_{i=1}^{k^*}\frac{N\lambda_i^2}{(N\lambda_i+\lambda)^2} + \sum_{i>k^*}\frac{N\lambda_i^2}{(1+\lambda)^2}\Bigg),
\end{align*}
where $k^* = \min\{k:N\lambda_k\le 1\}$.
\end{lemma}
\begin{proof}
In the one-hot case, it is easy to verify that $\Xb^\top\Xb = \sum_{i=1}^n\xb_i\xb_i^\top$ is a diagonal matrix. Let $\mu_1,\mu_2,\dots,\mu_d$ be the eigenvalues of $\Xb^\top\Xb$ corresponding to the eigenvectors $\eb_1, \eb_2,\dots,\eb_d$ respectively. Then by  Lemma \ref{lemma:bias_var_decomposition_ridge}, we have the following results for the bias and variance errors of ridge regression.
\begin{align}\label{eq:sum_expectation_bias_onehot_ridge}
\ridgebias &= \lambda^2\cdot\EE\big[\wb^{*\top}(\Xb^\top\Xb+\lambda\Ib)^{-1}\Hb(\Xb^\top\Xb+\lambda\Ib)^{-1}\wb^*\big]\notag\\
&= \lambda^2\sum_i \EE_{\mu_i}\bigg[\frac{\lambda_i\wb^*[i]^2}{(\mu_i+\lambda)^2}\bigg],
\end{align}
where the expectation in the first equation is taken over the training data $\Xb$ and in the second inequality the expectation is equivalently taken over the eigenvalues $\mu_1,\dots,\mu_d$. Since $\xb_i$ can only take on natural basis, the eigenvalue $\mu_i$ can be understood as the number of training data that equals $\eb_i$. Note that the probability of sampling $\eb_i$ is $\lambda_i$, then we can get that $\mu_i$ has a marginal distribution $\mathrm{Binom}(N,\lambda_i)$, where $N$ is the sample size. Then in terms of each expectation in \eqref{eq:sum_expectation_bias_onehot_ridge}, we first have
\begin{align*}
\EE_{\mu_i}\bigg[\frac{\lambda_i\wb^*[i]^2}{(\mu_i+\lambda)^2}\bigg]\ge \frac{\lambda_i\wb^*[i]^2}{(\EE[\mu_i]+\lambda)^2} = \frac{\lambda_i\wb^*[i]^2}{(N\lambda_i+\lambda)^2},
\end{align*}
where the first inequality is by applying Jensen's inequality to the convex function $f(x)=1/(x+\lambda)^2$. On the other hand, we also have
\begin{align*}
\EE_{\mu_i}\bigg[\frac{\lambda_i\wb^*[i]^2}{(\mu_i+\lambda)^2}\bigg]\ge \frac{\lambda_i\wb^*[i]^2}{\lambda^2}\cdot \PP(\mu_i=0) = \frac{\lambda_i\wb^*[i]^2}{\lambda^2}\cdot (1-\lambda_i)^N.
\end{align*}
Therefore, combining the above two lower bounds, we can get the following lower bound on the bias error by \eqref{eq:sum_expectation_bias_onehot_ridge}
\begin{align}\label{eq:sum_expectation_bias_onehot_ridge2}
\ridgebias = \lambda^2\sum_i \EE_{\mu_i}\bigg[\frac{\lambda_i\wb^*[i]^2}{(\mu_i+\lambda)^2}\bigg]\ge \sum_{i}\max\bigg\{\frac{\lambda^2\lambda_i\wb^*[i]^2}{(N\lambda_i+\lambda)^2}, \lambda_i\wb^*[i]^2\cdot(1-\lambda_i)^N\bigg\}.
\end{align}
Therefore, a trivial lower bound on the bias error of ridge regression is
\begin{align*}
\ridgebias\ge \sum_i (1-\lambda_i)^N\cdot\lambda_i\wb^*[i]^2.
\end{align*}
Additionally, note that $(1-\lambda_i)^N\ge 0.25$ if $\lambda_i\le 1/N$ and $N\ge 2$. Then let $k^* = \min\{k: N\lambda_k\le 1\}$, \eqref{eq:sum_expectation_bias_onehot_ridge2} further leads to
\begin{align*}
\ridgebias \ge \sum_{i=1}^{k^*} \frac{\lambda^2\lambda_i\wb^*[i]^2}{(N\lambda_i+\lambda)^2} + 0.25\cdot\sum_{i>k^*} \lambda_i\wb^*[i]^2.    
\end{align*}

This completes the proof of the lower bound of the bias error.

By Lemma \ref{lemma:bias_var_decomposition_ridge}, we have
\begin{align}\label{eq:sum_expectation_var_onehot_ridge}
\ridgevar &= \sigma^2\cdot\EE\big[\tr\big((\Xb^\top\Xb+\lambda\Ib)^{-1}\Xb^\top\Xb(\Xb^\top\Xb+\lambda\Ib)^{-1}\Hb\big)\big]\notag\\
&= \sigma^2\cdot\sum_i \EE_{\mu_i}\bigg[\frac{\lambda_i\mu_i}{(\mu_i+\lambda)^2}\bigg],
\end{align}
Regarding the variance error, we cannot use the similar approach since the function $g(x)=x/(x+\lambda)^2$ is no longer convex. Instead, we will directly make use of property of the binomial distribution of $\mu_i$ to prove the desired bound. In particular, note that $\mu_i\sim\mathrm{binom}(N,\lambda_i)$, by Bernstein inequality, we have
\begin{align*}
\PP(|\mu_i-N\lambda_i|\le t)\ge 1 - 2\exp\bigg(-\frac{t^2}{2(N\lambda_i+t/3)}\bigg).
\end{align*}
If $N\lambda_i\ge 6$, by set $t = \sqrt{3N\lambda_i}$, we have 
\begin{align*}
\PP\big(\mu_i\in\big[N\lambda_i-\sqrt{3N\lambda_i},N\lambda_i+\sqrt{3N\lambda_i}\big]\big)\ge 1 -2 e^{-1} \ge 0.2,
\end{align*}
which further implies that 
\begin{align*}
\PP\big(\mu_i\in\big[0.25N\lambda_i,2N\lambda_i\big]\big)\ge 0.2,
\end{align*}
where we use the fact that $\sqrt{3N\lambda_i}\le 0.75N\lambda_i$ if $N\lambda_i>6$. Therefore, in this case, we can get
\begin{align}\label{eq:bound_case1}
\EE_{\mu_i}\bigg[\frac{\lambda_i\mu_i}{(\mu_i+\lambda)^2}\bigg]\ge 0.2\min\bigg\{\frac{0.25N\lambda_i^2}{(0.25N\lambda_i+\lambda)^2},\frac{2N\lambda_i^2}{(2N\lambda_i+\lambda)^2}\bigg\}\ge \frac{0.05N\lambda_i^2}{(N\lambda_i+\lambda)^2}.
\end{align}
Then we consider the case that $N\lambda_i<6$. In particular, we have
\begin{align}\label{eq:lowerbound_expect_var_case2}
\EE_{\mu_i}\bigg[\frac{\lambda_i\mu_i}{(\mu_i+\lambda)^2}\bigg]\ge \frac{\lambda_i}{(1+\lambda)^2}\cdot \PP(\mu_i=1).
\end{align}
Note that $\mu_i$ follows $\text{Binom}(N,\lambda_i)$ distribution, which implies that 
\begin{align*}
\PP(\mu_i=1) = N\lambda_i (1-\lambda_i)^{N-1}\ge N\lambda_i \big(1-\frac{6}{N}\big)^{N-1}\ge e^{-6}N\lambda_i.
\end{align*}
Plugging this into \eqref{eq:lowerbound_expect_var_case2} gives
\begin{align}\label{eq:bound_case2}
\EE_{\mu_i}\bigg[\frac{\lambda_i\mu_i}{(\mu_i+\lambda)^2}\bigg]\ge\frac{e^{-6}N\lambda_i^2}{(1+\lambda)^2}.
\end{align}
Therefore, let $k^*=\min\{k:N\lambda_k\le 1\}$, then for all $i\le k^*$, combining \eqref{eq:bound_case1} and \eqref{eq:bound_case2} gives
\begin{align*}
\EE_{\mu_i}\bigg[\frac{\lambda_i\mu_i}{(\mu_i+\lambda)^2}\bigg]\ge \frac{e^{-6}N\lambda_i^2}{(N\lambda_i+\lambda)^2}.
\end{align*}
For all $i> k^*$, we can directly apply \eqref{eq:bound_case2} to get the lower bound. Therefore, according to \eqref{eq:sum_expectation_var_onehot_ridge}, the variance error can be lower bounded as follows,
\begin{align*}
\ridgevar 
&= \sigma^2\cdot\sum_i \EE_{\mu_i}\bigg[\frac{\lambda_i\mu_i}{(\mu_i+\lambda)^2}\bigg]\notag\\
&\ge e^{-6}\sigma^2\cdot\bigg(\sum_{i=1}^{k^*}\frac{N\lambda_i^2}{(N\lambda_i+\lambda)^2} + \sum_{i>k^*}\frac{N\lambda_i^2}{(1+\lambda)^2}\bigg).
\end{align*}
This completes the proof of the lower bound of the variance error.
\end{proof}

\subsection{Proof of Theorem \ref{thm:comparison_onehot}}
\begin{proof}
In the beginning, we first recall the excess risk upper bound of SGD (see Lemma \ref{lemma:excessrisk_sgd_onehot}) and excess risk lower bound of ridge (see Lemma \ref{lemma:excessrisk_sgd_onehot}) as follows,
\begin{align*}
\EE[L(\wb_{\mathrm{sgd}}(\Ns;\gamma))]-L(\wb^*) \le 2\cdot\sgdbias + 2\cdot\sgdvar,
\end{align*}
where 
\begin{align}\label{eq:bias_var_sgd_onehot}
\sgdbias &\lesssim \frac{1}{N^2\gamma^2}\cdot \big\|(\Ib-\gamma\Hb)^{N/2}\wb^*\big\|_{\Hb_{0:k_1}^{-1}} + \big\|(\Ib-\gamma\Hb)^{N/2}\wb^*\big\|_{\Hb_{k_1:\infty}}^2\notag\\
\sgdvar &\lesssim\sigma^2\cdot\bigg(\frac{k_2}{N} + N\gamma^2\sum_{i>k_2}\lambda_i^2\bigg)
\end{align}
for arbitrary $k_1,k_2\in[d]$.

\begin{align*}
\EE[L(\wb_{\mathrm{ridge}}(N;\lambda))]-L(\wb^*) = \ridgebias + \ridgevar,
\end{align*}
where
\begin{align}\label{eq:bias_var_ridge_onehot}
\ridgebias&\gtrsim \max\bigg\{\sum_i (1-\lambda_i)^N\cdot\lambda_i\wb^*[i]^2,\sum_{i=1}^{k^*} \frac{\lambda^2\lambda_i\wb^*[i]^2}{(N\lambda_i+\lambda)^2} +\sum_{i>k^*} \lambda_i\wb^*[i]^2\bigg\} \notag\\  
\ridgevar
&\gtrsim \sigma^2\cdot\Bigg(\sum_{i=1}^{k^*}\frac{N\lambda_i^2}{(N\lambda_i+\lambda)^2} + \sum_{i>k^*}\frac{N\lambda_i^2}{(1+\lambda)^2}\Bigg),
\end{align}
where $k^* = \min\{k:N\lambda_k\le 1\}$.

Next, we will show that the excess risk of SGD can be provably upper bounded (up to constant factors) by the excess risk of ridge regression respectively, given the sample size of ridge regression $\Nr$ (which we will use $N$ in the remaining proof for simplicity). In particular, we consider two cases regarding different $\lambda$: \textbf{Case I} $\lambda<1$ and \textbf{Case II} $\lambda\ge 1$.

For \textbf{Case I}, \eqref{eq:bias_var_ridge_onehot} gives the following bias lower bound for ridge regression,
\begin{align*}
\ridgebias &\gtrsim \sum_i (1-\lambda_i)^N\cdot\lambda_i\wb^*[i]^2\notag\\
&\gtrsim \sum_{i>k^*} \lambda_i\wb^*[i]^2 \notag\\  
\ridgevar&\gtrsim \sigma^2\cdot\Bigg(\sum_{i=1}^{k^*}\frac{N\lambda_i^2}{(N\lambda_i+\lambda)^2} + \sum_{i>k^*}\frac{N\lambda_i^2}{(1+\lambda)^2}\Bigg)\notag\\
&\overset{(i)}\eqsim\sigma^2\cdot\Bigg(\frac{k^*}{N} + N\sum_{i>k^*}\lambda_i^2\Bigg),
\end{align*}
where in $(i)$ we use the fact that $N\lambda_i+\lambda \eqsim N\lambda_i$ for all $i\le k^*$.

Then let $R^2=\|\wb^*\|_2^2/\sigma^2$ denotes the signal-to-noise ratio, let's consider the following configuration for SGD:
\begin{align*}
\Ns = N,\quad \gamma = 1. 
\end{align*}
Then by \eqref{eq:bias_var_sgd_onehot} and setting $k_1=0$ and $k_2=k^*$, we get
\begin{align*}
\sgdbias &\lesssim  \sum_{i}(1-\lambda_i)^N\cdot\lambda_i\wb^*[i]^2\notag\\
\sgdvar &\lesssim\sigma^2\cdot\bigg(\frac{k^*}{\Ns} + \Ns\gamma^2\sum_{i>k^*}\lambda_i^2\bigg)\notag\\
&\overset{(i)}\lesssim \sigma^2\cdot\bigg(\frac{k^*}{N} + N\sum_{i>k^*}\lambda_i^2\bigg).
\end{align*}
Therefore, given such choice of $\Ns$ and $\gamma$, we have
\begin{align*}
\EE[L(\wb_{\mathrm{sgd}}(\Ns;\gamma))]-L(\wb^*) &\lesssim \sgdbias + \sgdvar\notag\\
&\lesssim \sum_{i}(1-\lambda_i)^N\cdot\lambda_i\wb^*[i]^2 + \sigma^2\cdot\bigg(\frac{k^*}{N} + N\sum_{i>k^*}\lambda_i^2\bigg)\notag\\
&\lesssim \ridgebias + \ridgevar\notag\\
&=\EE[L(\wb_{\mathrm{ridge}(N;\lambda)})]-L(\wb^*).
\end{align*}

For \textbf{Case II}, we can define $\tilde k^*=\min\{k:N\lambda_k\le\lambda\}$, then \eqref{eq:bias_var_ridge_onehot} implies 
\begin{align*}
\ridgebias &\gtrsim \sum_{i=1}^{k^*} \frac{\lambda^2\lambda_i\wb^*[i]^2}{(N\lambda_i+\lambda)^2} +\sum_{i>k^*} \lambda_i\wb^*[i]^2 \notag\\  
&\overset{(i)}\eqsim \sum_{i=1}^{\tilde k^*} \frac{\lambda^2\wb^*[i]^2}{N^2\lambda_i} +\sum_{i>\tilde k^*} \lambda_i\wb^*[i]^2\notag\\
\ridgevar&\gtrsim \sigma^2\cdot\Bigg(\sum_{i=1}^{k^*}\frac{N\lambda_i^2}{(N\lambda_i+\lambda)^2} + \sum_{i>k^*}\frac{N\lambda_i^2}{(1+\lambda)^2}\Bigg)\notag\\
&\overset{(ii)}\eqsim\sigma^2\cdot\Bigg(\frac{\tilde k^*}{N} + \frac{N}{\lambda^2}\sum_{i>\tilde k^*}\lambda_i^2\Bigg),
\end{align*}
where $(i)$ and $(ii)$ are due to the fact that for every $i\le k^*$, we have
\begin{align*}
\frac{1}{(N\lambda_i+\lambda)^2}\eqsim \left\{
\begin{array}{ll}
  \frac{1}{N^2\lambda_i}   &  i\le \tilde k^*\\
   \frac{1}{\lambda^2}  &  \tilde k^*<i\le k^*.
\end{array}
\right.
\end{align*}
Therefore, we can apply the following configuration for SGD:
\begin{align*}
\Ns = N, \quad \gamma = 1/\lambda.
\end{align*}
Then by \eqref{eq:bias_var_sgd_onehot} and set $k_1=k_2=\tilde k^*$, we have
\begin{align*}
&\EE[L(\wb_{\mathrm{sgd}}(\Ns;\gamma))]-L(\wb^*) \notag\\
&\lesssim \sgdbias + \sgdvar\notag\\
& \lesssim \sum_{i=1}^{\tilde k^*}\frac{(1-\gamma\lambda_i)^{\Ns}\wb^*[i]^2}{\lambda_i\Ns^2\gamma^2} + \sum_{i>\tilde k^*}\lambda_i\wb^*[i]^2 + \sigma^2\cdot\bigg(\frac{\tilde k^*}{\Ns} + \Ns\gamma^2\sum_{i>\tilde k^*}\lambda_i^2\bigg)\notag\\
&\eqsim \sum_{i=1}^{\tilde k^*}\frac{\lambda^2\wb^*[i]^2}{\lambda_iN^2} + \sum_{i>\tilde k^*}\lambda_i\wb^*[i]^2 + \sigma^2\cdot\bigg(\frac{\tilde k^*}{N} + \frac{N}{\lambda^2}\sum_{i>\tilde k^*}\lambda_i^2\bigg)\notag\\
&\lesssim \ridgebias + \ridgevar\notag\\
&=\EE[L(\wb_{\mathrm{ridge}}(N;\lambda))]-L(\wb^*).
\end{align*}

Combining the results for these two cases completes the proof.
\end{proof}

\subsection{Proof of Theorem \ref{thm:SGD>ridge_onehot}}

\begin{proof}
For simplicity we define $N:=\Ns$ in the proof.
\begin{itemize}
    \item The data covariance matrix $\Hb$ has the following spectrum
\begin{align*}
\lambda_i = 
\begin{cases}
    \frac{\log(N)}{N^{1/2}} &  i=1,\\
    \frac{1-\log(N)/N^{1/2}}{N} & 1< i \le N, \\
    0 & N < i \le d
\end{cases}
\end{align*}
\item The true parameter $\wb^*$ is given by 
\begin{align*}
\wb^*[i] = 
\begin{cases}
    \sigma\cdot\sqrt{\frac{N^{1/2}}{\log(N)}}  &  i=1,\\
    0 & 1<i\le d.
\end{cases}
\end{align*}
\end{itemize}

 Then it is easy to verify that $\tr(\Hb)=1$. For SGD, we consider setting the stepsize as $\gamma^*=N^{-1/2}$. Then by Lemma \ref{lemma:excessrisk_sgd_onehot} and choosing $k_1=1$, we have the following on the bias error of SGD,
\begin{align*}
\sgdbias &\lesssim \sum_{i=1}^{k^*}\frac{(1-\gamma\lambda_i)^{\Ns}\wb^*[i]^2}{\lambda_i\Ns^2\gamma^2} + \sum_{i>k^*}\lambda_i\wb^*[i]^2\lesssim \frac{(1-\log(N)/N)^{N}\sigma^2}{\log^2(N)} \lesssim \frac{\sigma^2}{N}.
\end{align*}
For variance error, we can pick $k_2=1$ and get
\begin{align*}
\sgdvar &\lesssim\sigma^2\cdot\bigg(\frac{1}{N} + N\gamma^2\sum_{i>1}\lambda_i^2\bigg) \lesssim \sigma^2\bigg(\frac{1}{N} + \sum_{i>1}\lambda_i^2\bigg)\eqsim \frac{\sigma^2}{N}.
\end{align*}

Now let us characterize the excess risk of ridge regression. In terms of the bias error, by Lemma \ref{lemma:excessrisk_ridge_onehot} we have
\begin{align}\label{eq:bound_ridgebias_example_onehot}
\ridgebias \gtrsim \sum_{i=1}^{k^*}\frac{\lambda^2\lambda_i\wb^*[i]^2}{(\Nr\lambda_i+\lambda)^2} + \sum_{i>k^*}\lambda_i\wb^*[i]^2 \eqsim \frac{\lambda^2\sigma^2}{(\Nr\log(N)/N^{1/2}+\lambda)^2},
\end{align}
where $k^*=\min\{k:\Nr\lambda_k\le 1\}$. Then it is clear for ridge regression we must have $\lambda\lesssim \Nr\log(N)/N^{1/2}$ since otherwise $\ridgebias\gtrsim \sigma^2\gtrsim \mathrm{SGDRisk}$. Regarding the variance, we have
\begin{align*}
\ridgevar 
&\gtrsim \sigma^2\cdot\Bigg(\sum_{i=1}^{k^*}\frac{\Nr\lambda_i^2}{(\Nr\lambda_i+\lambda)^2} + \sum_{i>k^*}\frac{\Nr\lambda_i^2}{(1+\lambda)^2}\Bigg).
\end{align*}
Then we will consider two cases: (1) $\Nr\lesssim N$ and (2)
$\Nr\gtrsim N$. In the first case we can get $k^*=1$ and then 
\begin{align*}
\ridgevar 
&\gtrsim \sigma^2\cdot\Bigg(\frac{\Nr\log^2(N)/N^2}{(\Nr\log(N)/N+\lambda)^2} + \frac{\Nr}{N^2(1+\lambda)^2}\Bigg)\ge \frac{\Nr\sigma^2}{N^2(1+\lambda^2)}.
\end{align*}
In this case, we can get $k^*=1$ and thus
\begin{align*}
\ridgevar 
&\gtrsim \sigma^2\cdot\frac{\Nr\log^2(N)/N^2}{(\Nr\log(N)/N+\lambda)^2}\overset{(i)}\eqsim \frac{\sigma^2}{\Nr}\overset{(ii)}\gtrsim\frac{\sigma^2}{N},
\end{align*}
where $(i)$ is due to we require $\lambda\lesssim \Nr\log(N)/N^{1/2}$ to guarantee vanishing bias error and $(ii)$ is due to in this case we have $\Nr\lesssim N$. As a result, ridge regression cannot achieve smaller excess risk than SGD in this case.

In the second case we can get $k^*=N$ and then
\begin{align}\label{eq:bound_ridgevar_example_onehot}
\ridgevar 
&\gtrsim \sigma^2\cdot\bigg(\frac{\Nr\log^2(N)/N^2}{(\Nr\log(N)/N+\lambda)^2} + \frac{(k^*-1)\cdot\Nr/N^2}{(\Nr/N+\lambda^2)}\bigg)\notag\\
&\gtrsim \sigma^2\cdot\frac{N\Nr}{\Nr^2 + N^2\lambda^2},
\end{align}
where the second inequality is due to $k^*=N$. We will again consider two cases: (a) $\Nr\gtrsim N\lambda$ and (b) $\Nr\lesssim N\lambda$. Regarding Case (a) we have
\begin{align*}
\ridgevar\ge \frac{N\sigma^2}{\Nr},
\end{align*}
and it is clear that for all $\Nr\lesssim N^2$ we have $\ridgevar\gtrsim \sigma^2/N\gtrsim \mathrm{SGDRisk}$. Regarding Case (b), combining the lower bounds of bias \eqref{eq:bound_ridgebias_example_onehot} and variance \eqref{eq:bound_ridgevar_example_onehot} of ridge regression, we get
\begin{align*}
\mathrm{RidgeRisk} \gtrsim \sigma^2\cdot\bigg(\frac{\lambda^2N}{\Nr^2\log^2(N)}+\frac{\Nr}{N\lambda^2}\bigg)\gtrsim \frac{\sigma^2}{\Nr^{1/2}\log(N)},
\end{align*}
where the first inequality follows from the fact that $\lambda\lesssim \Nr\log(N)/N^{1/2}$ and $\Nr\lesssim N\lambda$, and  the second inequality is by  Cauchy-Schwartz inequality. This further suggests that $\mathrm{RidgeRisk}\lesssim \sigma^2/N\lesssim \mathrm{SGDRisk}$ if $\Nr\le N^2/\log^2(N)$, which completes the proof.

\end{proof}

\section{Proof of Gaussian Least Squares}\label{sec:proof_main_full}

\subsection{Excess risk bounds of SGD and ridge regression}

We first recall the excess risk bounds for SGD (with tail averaging) and ridge regression as follows.

\noindent\textbf{SGD with tail averaging} 
\begin{theorem}[Extension of Theorem 5.1 in \citet{zou2021benign}]\label{thm:generalization_error_tail}
Consider SGD with tail-averaging with initialization $\wb_0=\bm{0}$.
Suppose Assumption~\ref{assump:data_distribution} holds
and the stepsize satisfies $\gamma\lesssim 1/\tr(\Hb)$.
Then the excess risk can be upper bounded as follows,
\begin{align*}
\EE [L(\wb_{\mathrm{sgd}}(N;\gamma))] - L(\wb^*)
&\le  \sgdbias + \sgdvar,
\end{align*}
where 
\begin{align*}
 \sgdbias & \lesssim \frac{1}{\gamma^2N^2}\cdot\big\|(\Ib-\gamma\Hb)^N\wb^*\big\|_{\Hb_{0:k_1}^{-1}}^2 + \big\|(\Ib-\gamma\Hb)^N\wb^*\big\|_{\Hb_{k_1:\infty}}^2 \\
 \sgdvar & \lesssim \frac{\sigma^2 + \|\wb^*\|_\Hb^2}{N}\cdot\bigg(k_2 + N^2\gamma^2 \sum_{i> k_2}\lambda_i^2\bigg).
\end{align*}
where $k_1,k_2\in[d]$ are arbitrary.
\end{theorem}
This theorem is a simple extension of Theorem 5.1 in \citet{zou2021benign}. In particular, we observe that though the original theorem is stated for some particular $k^*$ and $k^\dagger$, based on the proof, their results hold for arbitrary $k_1$ and $k_2$, as stated in Theorem \ref{thm:generalization_error_tail}.

\noindent\textbf{Ridge regression.}
See Appendix \ref{appendix:proof_ridge} for a proof of the following theorem.

\begin{theorem}[Extension of Lemmas 2 \& 3 in \citet{tsigler2020benign}]\label{thm:lowerbound_ridge}
 Suppose Assumption \ref{assump:data_distribution} holds. Let $\lambda\ge 0$ be the regularization parameter, $n$ be the training sample size and $\hat\wb_{\mathrm{ridge}}(N;\lambda)$ be the output of ridge regression. 
 Then 
\begin{align*}
\EE\big[L(\wb_{\mathrm{ridge}}(N;\lambda))\big]-L(\wb^*) = \ridgebias +  \ridgevar ,
\end{align*}
and there is some absolute constant $b > 1$, such that for
\[k^*_{\mathrm{ridge}} := \min\left\{k: b  \lambda_{k+1} \le  \frac{\lambda+\sum_{i>k}\lambda_i}{n }\right\}, \]
the following holds:
\begin{align*}
\ridgebias &\gtrsim\bigg(\frac{\lambda+\sum_{i>\kr}\lambda_i}{N}\bigg)^2\cdot\|\wb^*\|_{\Hb_{0:\kr}^{-1}}^2+\|\wb^*\|_{\Hb_{\kr:\infty}}^2,\notag\\
\ridgevar &\gtrsim\sigma^2\cdot\bigg\{\frac{\kr}{N}+\frac{N\sum_{i>\kr}\lambda_i^2}{\big(\lambda+\sum_{i>\kr}\lambda_i\big)^2}\Bigg\}.
\end{align*}
\end{theorem}

\subsection{Proof of Theorem \ref{thm:SGD<ridge}}
\begin{proof}
For simplicity, let us fix $N := N_{\mathrm{ridge}}$ and  $k := k_{\mathrm{ridge}}$, we will next locate $\gamma$ such that the risk of SGD competes with that of Ridge.
Denote $\tilde{\lambda} := \lambda + \sum_{i>k} \lambda_i$. Then 
\begin{align*}
    \mathrm{RidgeRisk}&= \ridgebias+\ridgevar \notag\\
    &\gtrsim \rbr{\frac{\tilde{\lambda}}{N}}^2 \nbr{\wb^*}^2_{\Hb^{-1}_{0:k}} + \nbr{\wb^*}^2_{\Hb_{k:\infty}}
    + \frac{\sigma^2}{N} \rbr{k + \rbr{\frac{N}{\tilde{\lambda}}}^2 \sum_{i>k}\lambda_i^2 }.
\end{align*}

Then for SGD we can set
\[
N_{\mathrm{sgd}} = (1+R^2) \cdot N \cdot (1\lor \kappa\log a),
\]
where 
\[
\kappa := \frac{\tr(\Hb)}{N \lambda_N},\qquad a = \frac{\tr(\Hb)}{\lambda + \sum_{i>N}\lambda_i } \land (\kappa R \sqrt{N}) = \frac{\tr(\Hb)}{\lambda + \sum_{i>N}\lambda_i } \land \frac{\tr(\Hb) R}{ \sqrt{N} \lambda_N }.
\]
Next we discuss two cases:

\paragraph{Case I, $\tilde{\lambda}\cdot (1 \lor \kappa \log a) \ge \tr (\Hb)$.}
For SGD, let us set $k_{\mathrm{sgd}} = k $ and that
\begin{align*}
    \gamma = \frac{1}{(1+R^2)\cdot \tilde{\lambda}\cdot (1\lor \kappa \log a) } \le \frac{1}{\tr(\Hb)},
\end{align*}
then 
\[ N_{\mathrm{sgd}} \cdot \gamma = \frac{N}{ \tilde{\lambda} }.
\]
Thus we obtain that 
\begin{align*}
    \mathrm{SGDRisk}
    & \lesssim \frac{(1-\gamma \lambda_k)^{2N_{\mathrm{sgd}}}}{\rbr{\gamma N_{\mathrm{sgd}}}^2} \nbr{\wb^*}^2_{\Hb^{-1}_{0:k}} + \nbr{\wb^*}^2_{\Hb_{k:\infty}}
    + \frac{(1+R^2)\sigma^2}{N_{\mathrm{sgd}}} \rbr{k + \rbr{\gamma N_{ \mathrm{sgd}}}^2 \sum_{i>k}\lambda_i^2 } \\
    &= \frac{(1-\gamma \lambda_k)^{2N_{\mathrm{sgd}}}}{\rbr{N / \tilde{\lambda}}^2} \nbr{\wb^*}^2_{\Hb^{-1}_{0:k}} + \nbr{\wb^*}^2_{\Hb_{k:\infty}}
    + \frac{\sigma^2}{N (1 \lor \kappa \log a) } \rbr{k + \rbr{N / \tilde{\lambda}}^2 \sum_{i>k}\lambda_i^2 } \\
    &\le \rbr{\frac{\tilde{\lambda}}{N}}^2 \nbr{\wb^*}^2_{\Hb^{-1}_{0:k}} + \nbr{\wb^*}^2_{\Hb_{k:\infty}}
    + \frac{\sigma^2}{N} \rbr{k + \rbr{\frac{N}{\tilde{\lambda}}}^2 \sum_{i>k}\lambda_i^2 } \\
    &\lesssim \mathrm{RidgeRisk} .
\end{align*}

\paragraph{Case II, $\tilde{\lambda} \cdot(1\lor \kappa \log a) < \tr (\Hb)$.}
For SGD, 
let us set $k_{\mathrm{sgd}} = k$ and that
\begin{align*}
    \gamma = \frac{1}{ (1+R^2)\cdot\tr(\Hb) } \le \frac{1}{\tr(\Hb)},
\end{align*}
then 
\[ N_{\mathrm{sgd}} \cdot \gamma = \frac{N \cdot (1\lor \kappa \log a)  }{\tr(\Hb)} \le \frac{N} { \tilde{\lambda} }. \]
We obtain that 
\begin{align*}
    \mathrm{SGDRisk}
&\le  \sgdbias + \sgdvar\notag\\
    & \lesssim \frac{(1-\gamma \lambda_k)^{2N_{\mathrm{sgd}}}}{\rbr{\gamma N_{\mathrm{sgd}}}^2} \nbr{\wb^*}^2_{\Hb^{-1}_{0:k}} + \nbr{\wb^*}^2_{\Hb_{k:\infty}}
    + \frac{(1+R^2)\sigma^2}{N_{\mathrm{sgd}}} \rbr{k + \rbr{\gamma N_{ \mathrm{sgd}}}^2 \sum_{i>k}\lambda_i^2 } \\
    &\le \frac{(1-\gamma \lambda_k)^{2N_{\mathrm{sgd}}}}{\rbr{\gamma N_{\mathrm{sgd}}}^2} \nbr{\wb^*}^2_{\Hb^{-1}_{0:k}} + \nbr{\wb^*}^2_{\Hb_{k:\infty}}
    + \frac{\sigma^2}{N (1 \lor \kappa \log a) } \rbr{k + \rbr{N / \tilde{\lambda}}^2 \sum_{i>k}\lambda_i^2 } \\
    &\le \frac{(1-\gamma \lambda_k)^{2N_{\mathrm{sgd}}}}{\rbr{\gamma N_{\mathrm{sgd}}}^2} \nbr{\wb^*}^2_{\Hb^{-1}_{0:k}} + \nbr{\wb^*}^2_{\Hb_{k:\infty}}
    + \frac{\sigma^2}{N} \rbr{k + \rbr{\frac{N}{\tilde{\lambda}}}^2 \sum_{i>k}\lambda_i^2 }.
\end{align*}
The second and the third terms match those of ridge error. As for the first term, notice that by the choice of $\gamma$ and that $\lambda_k \ge \lambda_N$, we have that
\begin{align*}
    \frac{(1-\gamma \lambda_k)^{N_{\mathrm{sgd}}}}{\gamma N_{\mathrm{sgd}}}
    &\le \rbr{ 1-\frac{\lambda_N}{(1+R^2)\cdot \tr(\Hb)}}^{N_{\mathrm{sgd}} }\cdot \frac{ 1 }{\gamma N_{\mathrm{sgd}}} \\
    &=  \rbr{ 1-\frac{1}{(1+R^2) \cdot N\cdot \kappa }}^{ (1+R^2) \cdot N \cdot  (1 \lor \kappa\log a) } \cdot  \frac{\tr(\Hb) }{N \cdot (1\lor \kappa \log a)} \\
    &\le \rbr{ 1-\frac{1}{(1+R^2) \cdot N\cdot \kappa }}^{ (1+R^2)\cdot N\cdot\kappa    \log a } \cdot  \frac{\tr(\Hb) }{N } \\
    &\le \frac{1}{a} \cdot \frac{\tr(\Hb)}{ N } 
    = \frac{ ( \lambda + \sum_{i>N}\lambda_i) \lor (\sqrt{N}\lambda_N / R) }{\tr(\Hb)} \cdot \frac{\tr(\Hb)}{ N } \\
    &\le \frac{\lambda + \sum_{i>k}\lambda_i}{N} \lor \frac{\lambda_k}{R\cdot\sqrt{N}}
    = \frac{\tilde{\lambda}}{N}\lor \frac{\lambda_k}{R \cdot \sqrt{N}}.
\end{align*}
If $\frac{(1-\gamma \lambda_k)^{N_{\mathrm{sgd}}}}{\gamma N_{\mathrm{sgd}}} \le \frac{\tilde{\lambda}}{N}$, then
\begin{align*}
    \mathrm{SGDRisk}
    & \lesssim  \frac{(1-\gamma \lambda_k)^{2N_{\mathrm{sgd}}}}{\rbr{\gamma N_{\mathrm{sgd}}}^2} \nbr{\wb^*}^2_{\Hb^{-1}_{0:k}} + \nbr{\wb^*}^2_{\Hb_{k:\infty}}
    + \frac{\sigma^2}{N} \rbr{k + \rbr{\frac{N}{\tilde{\lambda}}}^2 \sum_{i>k}\lambda_i^2 }\\
    &\le  \rbr{\frac{\tilde{\lambda}}{N}}^2 \nbr{\wb^*}^2_{\Hb^{-1}_{0:k}} + \nbr{\wb^*}^2_{\Hb_{k:\infty}}
    + \frac{\sigma^2}{N} \rbr{k + \rbr{\frac{N}{\tilde{\lambda}}}^2 \sum_{i>k}\lambda_i^2 } \\
    &\lesssim \mathrm{RidgeRisk} .
\end{align*}
If  $\frac{(1-\gamma \lambda_k)^{N_{\mathrm{sgd}}}}{\gamma N_{\mathrm{sgd}}} \le \frac{\lambda_k}{R\cdot\sqrt{N}}$, then
\[
\frac{(1-\gamma \lambda_k)^{2N_{\mathrm{sgd}}}}{\rbr{\gamma N_{\mathrm{sgd}}}^2} \nbr{\wb^*}^2_{\Hb^{-1}_{0:k}}
\le \frac{\lambda_k^2}{R^2\cdot N} \nbr{\wb^*}^2_{\Hb^{-1}_{0:k}}
\le  \frac{\nbr{\wb^*}^2_{\Hb} }{R^2 \cdot N} \le \frac{\sigma^2}{N},
\]
and
\begin{align*}
    \mathrm{SGDRisk}
    & \lesssim  \frac{(1-\gamma \lambda_k)^{2N_{\mathrm{sgd}}}}{\rbr{\gamma N_{\mathrm{sgd}}}^2} \nbr{\wb^*}^2_{\Hb^{-1}_{0:k}} + \nbr{\wb^*}^2_{\Hb_{k:\infty}}
    + \frac{\sigma^2}{N} \rbr{k + \rbr{\frac{N}{\tilde{\lambda}}}^2 \sum_{i>k}\lambda_i^2 }\\
    &\le \frac{\sigma^2}{N} + \nbr{\wb^*}^2_{\Hb_{k:\infty}}
    + \frac{\sigma^2}{N} \rbr{k + \rbr{\frac{N}{\tilde{\lambda}}}^2 \sum_{i>k}\lambda_i^2 } \\
    &\lesssim 2\cdot\mathrm{RidgeRisk} .
\end{align*}
These complete the proof.
\end{proof}

\subsection{Proof of Corollary \ref{cor:SGD<ridge}}
\begin{proof}
By Theorem \ref{thm:SGD<ridge}, we only need to verify that $\kappa(\Nr) \lesssim \log (\Nr)$.
Recall that $\lambda_i = 1/i^\alpha$ for $0 < \alpha \le 1$, and $d \lesssim \Nr$.
For $\alpha = 1$, then \[\tr(\Hb) = \sum_{i=1}^d i^{-\alpha} \lesssim \log d \lesssim \log (\Nr),\] 
thus 
\[\kappa(\Nr)
= \frac{\tr(\Hb)}{\Nr \lambda_{\min\{d, \Nr\}}} \lesssim \frac{\log (\Nr)}{ \Nr \cdot \Nr^{-1}} = \log (\Nr).
\]
For $\alpha < 1$, then \[\tr(\Hb) = \sum_{i=1}^d i^{-\alpha} \lesssim d^{1-\alpha} \lesssim  \Nr^{1-\alpha},\] 
thus 
\[\kappa(\Nr)
= \frac{\tr(\Hb)}{\Nr \lambda_{\{\Nr,d\}}} \lesssim \frac{\Nr^{1-\alpha}}{ \Nr \cdot \Nr^{-\alpha}} = 1.
\]

\end{proof}

\subsection{Proof of Corollary \ref{cor:SGD<ridge_random}}
\begin{proof}
Note that given random $\wb^*$, the expected risk considered in our paper will be including the expectation over both random data $\xb$ and random ground-truth $\wb^*$. Since the distribution of $\wb^*$ is rotation invariant, the expectation of $\wb^*[i]$ will be the same for all $i\in[d]$. Therefore, let $B = \EE[(\wb^*[i])^2]$, the following holds according to \eqref{eq:ridgelowerbound}
\begin{align*}
\mathrm{RidgeRisk}&\gtrsim \bigg(\frac{\tilde\lambda}{\Nr}\bigg)^2\cdot \EE\big[\|\wb^*\|_{\Hb_{0:k^*}^{-1}}^2\big] + \EE\big[\|\wb^*\|_{\Hb_{k^*:\infty}}^2\big] + \sigma^2\cdot\bigg(\frac{k^*}{\Nr}+\frac{\Nr}{\tilde\lambda^2}\sum_{i>k^*}\lambda_i^2\bigg)\notag\\
&=B\bigg(\frac{\tilde\lambda}{\Nr}\bigg)^2\cdot\sum_{i=1}^{k^*}i^\alpha + B\cdot\sum_{i=k^*+1}i^{-\alpha} + \sigma^2\cdot\bigg(\frac{k^*}{\Nr}+\frac{\Nr}{\tilde\lambda^2}\sum_{i>k^*}\lambda_i^2\bigg)
\end{align*}
where $k^*=\min\{k:\Nr\lambda_k\le \tilde\lambda\}$. Then note that $\lambda_i = i^{-\alpha}$, we have $k^* = (\Nr/\tilde\lambda)^{1/\alpha}$, which implies that  
\begin{align*}
\mathrm{RidgeRisk}&\gtrsim B\bigg(\frac{\tilde\lambda}{\Nr}\bigg)^2\cdot (k^*)^{1+\alpha} + B\cdot \big[d^{1-\alpha} - (k^*)^{1-\alpha}\big] + \sigma^2\cdot\bigg(\frac{k^*}{\Nr}+\frac{\Nr}{\tilde\lambda^2}\sum_{i>k^*}\lambda_i^2\bigg)\notag\\
&\gtrsim \Nr^{1-\alpha}\cdot B
\end{align*}
where we use the fact that $d=\Theta(N)$. Note that constant SNR $R=\Theta(1)$ implies that
\begin{align*}
\sigma^2 \eqsim B \sum_{i=1}^d\lambda_i \eqsim \Nr^{1-\alpha} B.  
\end{align*} 
Then by \eqref{eq:sgdupperbound} and set $\Ns = \Nr=N$ and $k_1=k_2=\Nr$, we have
\begin{align*}
\mathrm{SGDRisk}
&\lesssim \frac{1}{\gamma^2\Ns^2}\cdot\EE\big[\big\|\exp(-\Ns\gamma\Hb )\wb^*\big\|_{\Hb_{0:k_1}^{-1}}^2\big] + \EE\big[\|\wb^*\big\|_{\Hb_{k_1:\infty}}^2\big] \notag\\
&\qquad + (1+R^2)\sigma^2\cdot\bigg(\frac{k_2}{\Ns}+\Ns\gamma^2\sum_{i>k_2}\lambda_i^2\bigg)\notag\\
& = \frac{1}{\gamma^2N^2}\cdot\EE\big[\big\|\wb^*\big\|_{\Hb_{0:N}^{-1}}^2\big] + \EE\big[\|\wb^*\big\|_{\Hb_{N:d}}^2\big]  +BN^{1-\alpha}\cdot\bigg(1+N\gamma^2\sum_{i>N}\lambda_i^2\bigg).
\end{align*}
Note that we have
\begin{align*}
\EE\big[\big\|\wb^*\big\|_{\Hb_{0:N}^{-1}}^2\big] = BN^{1+\alpha},\  \EE\big[\|\wb^*\big\|_{\Hb_{N:d}}^2\big]=BN^{1-\alpha}.
\end{align*}
Then we can set $\gamma \eqsim 1/\tr(\Hb)\eqsim N^{\alpha -1}$ and get
\begin{align*}
\mathrm{SGDRisk}
& \lesssim \frac{B}{N^{2\alpha}}\cdot N^{1+\alpha} + BN^{1-\alpha} +BN^{1-\alpha}\cdot\bigg(1+N\gamma^2\sum_{i>N}\lambda_i^2\bigg)\notag\\
&\lesssim BN^{1-\alpha}\\
&\lesssim \mathrm{RidgeRisk}.
\end{align*}
This implies that SGD can be no worse than ridge regression as long as provided same or larger sample size, which completes the proof.

\end{proof}

\subsection{Proof of Theorem~\ref{thm:bestcase_gaussian}}
\begin{proof}
For simplicity we fix $N := \Ns$.
Let us consider the following problem instance:
\begin{itemize}
    \item The data covariance matrix $\Hb$ has the following spectrum
\begin{align*}
\lambda_i = 
\begin{cases}
    1 &  i=1,\\
    \frac{1}{N\log N} & 1< i \le N^2, \\
    0 & N^2 < i \le d
\end{cases}
\end{align*}
where we require the dimension $d \ge N^2$. We note that $\tr(\Hb) = 1 + N/
\log N \eqsim N/\log N$.
\item The true parameter $\wb^*$ is given by 
\begin{align*}
\wb^*[i] = 
\begin{cases}
    \sigma &  i=1,\\
    0 & 1<i\le d.
\end{cases}
\end{align*}
\end{itemize}
Then for SGD, we choose stepsize as $\gamma = \log(N)/(2N) \le 1/\tr(\Hb)$. 
By Lemma \ref{thm:generalization_error_tail}, we have the following excess risk bound for $\wb_{\mathrm{sgd}}(N; \gamma^*)$,
\begin{align*}
L\big[\wb_{\mathrm{sgd}}(N; \gamma)\big] - L(\wb^*) \le \sgdbias + \sgdvar,
\end{align*}
where 
\begin{align*}
\sgdbias
&\lesssim \sigma^2 \cdot \frac{(1-\gamma)^{N}}{(\gamma N)^2}
\lesssim \sigma^2\cdot \log^2 N \cdot \bigg(1-\frac{\log N }{2N}\bigg)^{N} 
\lesssim \frac{\sigma^2 \log^2 N}{N^2} \lesssim \frac{\sigma^2}{N},\notag\\
\sgdvar&\lesssim \frac{\sigma^2 }{N}\cdot\bigg(1 + (N\gamma)^2\sum_{i>1}\lambda_i^2\bigg)\eqsim  \frac{\sigma^2 }{N},
\end{align*}
where we use the fact that $\sum_{i>1}\lambda_i^2 = \frac{1}{\log^2 N}$. This implies that SGD with sample size $N$ achieves at most $\cO(\sigma^2/N)$ excess risk on this example.

Then we calculate the excess risk lower bound of ridge regression. By Lemma \ref{thm:lowerbound_ridge} and let $\tilde\lambda = \lambda + \sum_{i>\kr}\lambda_i$, we have
\begin{align*}
L\big[\wb_{\mathrm{ridge}}(N; \lambda)\big] - L(\wb^*) 
&= \ridgebias + \ridgevar \\
&\gtrsim \sigma^2 \cdot \rbr{\frac{\tilde\lambda^2}{N_{\mathrm{ridge}}^2} + \frac{\kr}{\Nr} + \frac{\Nr\sum_{i>\kr}\lambda_i^2}{\tilde\lambda^2} }.
\end{align*}
If $\kr > N$, then 
\[
L\big[\wb_{\mathrm{ridge}}(N; \lambda)\big] - L(\wb^*) 
\gtrsim  \frac{\sigma^2 \kr}{\Nr} \ge \frac{\sigma^2 N}{\Nr} \ge \frac{\sigma^2 }{ N},\quad \text{ for } \Nr < \frac{N^2}{\log^2 N}.
\]
If $\kr \le N$, then $\sum_{i> \kr} \lambda_i^2 \ge \sum_{N < i \le N^2 } \frac{1}{N^2 \log^2 N} \eqsim \frac{1}{\log^2 N} $, which implies that 
\begin{align*}
L\big[\wb_{\mathrm{ridge}}(N; \lambda)\big] - L(\wb^*) 
&\gtrsim \sigma^2 \cdot \rbr{ \frac{\tilde\lambda^2}{N_{\mathrm{ridge}}^2} + \frac{\Nr}{\tilde\lambda^2}\cdot \frac{1}{\log^2 N} } \\
& \ge \frac{\sigma^2}{\Nr^{1/2} \log N} \\
&\ge \frac{\sigma^2}{N}, \quad \text{ for } \Nr < \frac{N^2}{\log^2 N}.
\end{align*}
To sum up, we have show that 
\begin{align*}
L\big[\wb_{\mathrm{ridge}}(\Nr; \lambda)\big] - L(\wb^*)\gtrsim \frac{\sigma^2 }{N} \gtrsim L\big[\wb_{\mathrm{sgd}}(N; \lambda)\big] - L(\wb^*), \quad \text{ for } \Nr < \frac{N^2}{\log^2 N}.
\end{align*}
This completes the proof.
\end{proof}

\subsection{Proof of Theorem \ref{thm:general_good_case}}
\begin{proof}
The proof of Theorem \ref{thm:general_good_case} is similar to that of Theorem \ref{thm:SGD<ridge}. In particular, we still consider two cases: (1) $\lambda\gtrsim \tr(\Hb)$ and (2) $\lambda\lesssim \tr(\Hb)$. For the first case, we can use the identical proof in Theorem \ref{thm:SGD<ridge} and get that SGD with sample size $N_{\mathrm{sgd}}\eqsim (1+R^2)\cdot N_{\mathrm{ridge}}$ to achieve better excess risk than ridge regression. Note that we have assumed $R^2=\Theta(1)$, therefore, we can claim that SGD outperforms ridge regression, as long as the sample size is at least in the same order of $N_{\mathrm{ridge}}$.

For the second case that $\lambda\lesssim \tr(\Hb)$, for simplicity we denote $N:=\Nr$ and 
we can directly set $\gamma =1/\tr(\Hb)$ and $\Ns= N$. Let $ k^*=\min\big\{k:\lambda_k\le\frac{\tr(\Hb)\log(N)}{N}\big\}$, then by the definition of $\kr$ in Lemma \ref{thm:lowerbound_ridge} and the assumption that ridge regression is in the generalizable regime, we have $k^*\le\kr\le \Nr$.
Therefore, applying Lemma \ref{thm:generalization_error_tail} with $k_1 = k^*$,
we have the following bound on the effective bias of SGD,
\begin{align*}
\sgdbias &\lesssim \sum_{i=1}^{k^*}\frac{(1-\gamma\lambda_i)^{N}(\wb^*[i])^2}{\lambda_i\gamma^2 N^2} + \sum_{i>k^*}\lambda_i(\wb[i])^2 \notag\\
&\lesssim \sum_{i=1}^{k^*}\frac{\big(1-\frac{\log(N)}{N}\big)^{N}(\wb^*[i])^2}{\lambda_i N^2} +  \sum_{i>k^*}\lambda_i(\wb[i])^2\notag\\
&\lesssim \frac{\|\wb^*\|_\Hb^2}{N} +  \sum_{i>k^*}\lambda_i(\wb[i])^2.
\end{align*}
Then by our assumption that 
\begin{align*}
\sum_{i=k^*}^{\Nr}\lambda_i\big(\wb[i]\big)^2\lesssim \frac{k^*\|\wb^*\|_\Hb^2}{N},
\end{align*}
we further have
\begin{align*}
\sgdbias&\lesssim \frac{\|\wb^*\|_\Hb^2}{N}+  \sum_{i>k^*}\lambda_i(\wb[i])^2\notag\\
&\lesssim  \sum_{i>\kr}\lambda_i(\wb[i])^2 + \frac{(\kr+1)\|\wb^*\|_\Hb^2}{N},
\end{align*}
where in the second inequality we use the fact that $k^*\le \kr\le \Nr$. Regarding the variance of SGD,  applying Lemma \ref{thm:generalization_error_tail} with $k_2=\kr$ gives
\begin{align*}
\sgdvar&\lesssim (\sigma^2+\|\wb^*\|_\Hb^2)\cdot\bigg(\frac{\kr}{N}+\frac{N }{(\tr(\Hb)^2)}\cdot\sum_{i\ge \kr}\lambda_i^2\bigg)\notag\\
&\lesssim (\sigma^2+\|\wb^*\|_\Hb^2)\cdot\bigg(\frac{\kr}{N}+\frac{N }{(\lambda + \sum_{i>\kr}\lambda_i)^2}\cdot\sum_{i\ge \kr}\lambda_i^2\bigg),
\end{align*}
where the last inequality is due to the fact that $\lambda\lesssim \tr(\Hb)$. Combining the above upper bounds for the  bias and variance of SGD, we have that the output of SGD, with sample size $\Ns=N$ and learning rate $\gamma=1/\tr(\Hb)$, satisfies
\begin{align}\label{eq:bound_SGD_goodcase}
\mathrm{SGDRisk}&\lesssim \sgdbias + \sgdvar\notag\\
&\lesssim \sum_{i>\kr}\lambda_i(\wb[i])^2 + \frac{(\kr+1)\|\wb^*\|_\Hb^2}{N}\notag\\
&\qquad+(\sigma^2+\|\wb^*\|_\Hb^2)\cdot\bigg(\frac{\kr}{\Nr}+\frac{\Nr \gamma^2}{(\lambda + \sum_{i>\Nr}\lambda_i)^2}\cdot\sum_{i\ge \kr}\lambda_i^2\bigg)\notag\\
&\eqsim \sum_{i>\kr}\lambda_i(\wb[i])^2 + \frac{(\kr+1)\|\wb^*\|_\Hb^2}{N}\notag\\
&\qquad+\sigma^2\cdot\bigg(\frac{\kr}{\Nr}+\frac{\Nr \gamma^2}{(\lambda + \sum_{i>\Nr}\lambda_i)^2}\cdot\sum_{i\ge \kr}\lambda_i^2\bigg)\notag\\
&\lesssim \ridgebias + \ridgevar,
\end{align}
where the last equality holds since we assume that $\|\wb\|_\Hb^2/\sigma^2=\Theta(1)$. Note that the R.H.S. of \eqref{eq:bound_SGD_goodcase} is exactly the lower bound of the excess risk of ridge regression. Therefore, we can conclude that as long as $\Ns=N$,  SGD with a tuned stepsize $\gamma$ will be no worse than ridge regression for all $\lambda$ (up to constant factors). This completes the proof.

\end{proof}

\section{Proof of Theorem \ref{thm:lowerbound_ridge}}\label{appendix:proof_ridge}
In this section we always make Assumption \ref{assump:data_distribution}.
The results and techniques are either explicitly or implicitly presented in \citep{bartlett2020benign,tsigler2020benign}. For self-completeness, we provide a formal proof here.

\paragraph{Notation.}
Following \citep{tsigler2020benign} and \citep{bartlett2020benign}, we define the following notations:
\begin{itemize}[leftmargin=*]
    \item $\vb := \Hb^{-\frac{1}{2}} \xb \in \RR^d$, then $\vb$ is sub-Gaussian and has independent components.
    \item Let $\Xb := (\xb_1, \dots, \xb_n )^\top \in \RR^{n\times d}$.
    Let $\Xb = \rbr{ \Xb_{0:k}\ \Xb_{k:\infty} }$ 
    \item Let $\Xb = \rbr{ \sqrt{\lambda_1} \zb_1, \dots, \sqrt{\lambda_d} \zb_d }\in \RR^{n\times d}$, then by Assumption \ref{assump:data_distribution}, $\zb_j$ is $1$-sub-Gaussian and has independent components.
    \item Let $\widetilde{\Ab} := \Xb\Xb^\top = \sum_{i=1}^d \lambda_i \zb_i \zb_i^\top \in \RR^{n \times n}$. Let $\Ab := \widetilde{\Ab} + \lambda_n \Ib_n = \Xb \Xb^\top + \lambda \Ib_n$.
    \item Let $\widetilde{\Ab}_k := \Xb_{k:\infty} \Xb_{k:\infty}^\top = \sum_{i\le k} \lambda_i \zb_i \zb_i^\top \in \RR^{n \times n}$. Let $\Ab_k :=\widetilde{\Ab}_k + \lambda \Ib_n = \Xb_{k:\infty} \Xb_{k:\infty}^\top + \lambda \Ib_n $.
    \item Let $\widetilde{\Ab}_{-j}:= \sum_{i \ne j} \lambda_i \zb_i \zb_i^\top \in \RR^{n \times n}$. Let ${\Ab}_{-j} := \widetilde{\Ab}_{-j} + \lambda \Ib_n$.
    \item Let $\rho_k := \frac{\lambda + \sum_{i>k} \lambda_i}{\lambda_{k+1}}$.
    \item Let $\Cb := \Ab^{-1} \Xb \Hb \Xb^\top \Ab^{-1}$.
    \item Let $\Bb := \rbr{\Ib_d - \Xb^\top  \Ab^{-1} \Xb} \Hb  \rbr{\Ib_d - \Xb^\top  \Ab^{-1} \Xb}$.
    \item We use $\EE_{\Xb}[\cdot]$ and $\EE_{\bepsilon}[\cdot]$ to denote the expectation with respect to the randomness of drawing $\Xb$ and the randomness of noise, respectively.
\end{itemize}

Under the above notations and from \citep{bartlett2020benign,tsigler2020benign}, we have 
\begin{equation*}
    \EE_{\Xb, \bepsilon} [ \text{ridge error} ] = \EE_{\Xb} [\ridgebias] + \EE_{\Xb,\bepsilon }[\ridgevar],
\end{equation*}
where 
\begin{equation*}
    \ridgebias := (\wb^*)^\top \Bb \wb^*,\qquad 
    \ridgevar := \bepsilon^\top \Cb \bepsilon.
\end{equation*}
We next provide lower bounds for each terms respectively.

\begin{lemma}[Variant of Lemma 10 in \citep{bartlett2020benign}]\label{lemma:ridge-eigenvalue-concentration}
There are constants $b, c \ge 1$ such that for every $k \ge 0$, with probability at least $0.1$, 
\begin{enumerate}
    \item for all $i \ge 1$, 
    \[
    \mu_{k+1}(\Ab_{-i}) \le \mu_{k+1} (\Ab) \le \mu_1(\Ab_k) \le c \rbr{ \lambda + \sum_{j > k} \lambda_j + \lambda_{k+1} n },
    \]
    \item for all $1 \le i \le k$, 
    \[
    \frac{1}{c} \rbr{\lambda + \sum_{j > k} \lambda_j} - c\lambda_{k+1} n \le \mu_n(\Ab_k) \le \mu_n(\Ab_{-i}) \le \mu_n(\Ab),
    \]
    \item if $\rho_k \ge b n$, then 
    \[
    \frac{1}{c} \lambda_{k+1} \rho_k \le \mu_n(\Ab_k) \le \mu_1(\Ab_k) \le c \lambda_{k+1} \rho_k.
    \]
    \item if $\rho_k \ge b n$, then for all $i > k$,
    \[
    \mu_n (\Ab_{-i}) \ge \frac{1}{c} \lambda_{k+1} \rho_k
    \]
\end{enumerate}
\end{lemma}
\begin{proof}
The first two claims are proved by noticing that $\Ab = \lambda \Ib + \widetilde{\Ab}$, $\Ab_k = \lambda \Ib + \widetilde{\Ab}_k$, $\Ab_{-i} = \lambda \Ib + \widetilde{\Ab}_{-i}$, and applying Lemma 10 in \citep{bartlett2020benign} to $\widetilde{\Ab}, \widetilde{\Ab}_k, \widetilde{\Ab}_{-j}$.

The third claim is proved by using the first two claims and that $\rho_k \ge bn$ to obtain that
\begin{align*}
    &\mu_1({\Ab}_k) \le c \rbr{\lambda +  \sum_{i>k} \lambda_i + \lambda_{k+1} n }
    \le \rbr{c + \frac{c}{b} } \cdot \rbr{\lambda + \sum_{i > k}\lambda_i }, \\
    & \mu_{n} ({\Ab}_k) \ge \frac{1}{c}\rbr{\lambda +  \sum_{i>k} \lambda_i} - c \lambda_{k+1} n \ge \rbr{\frac{1}{c} - \frac{c}{b}} \cdot \rbr{\lambda + \sum_{i > k}\lambda_i },
\end{align*}
and by re-scaling the constants.

The fourth claim is used in Lemma 3 in \citep{tsigler2020benign}, which can be proved under Assumption \ref{assump:data_distribution} as follows.
Let $i>k $ and $\widetilde{\Ab}_{k, -i} = \sum_{ j  > k, j\ne i} \lambda_j \zb_j \zb_j^\top$.
Then by Lemma 10 in \cite{bartlett2020benign} there is an absolute constant $c \ge 1$ such that 
\[
\mu_n(\widetilde{\Ab}_{-i} ) \ge \mu_n(\widetilde{\Ab}_{k, -i}) \ge \frac{1}{c} \sum_{j > k, j \ne i} \lambda_j - c \lambda_{k+1} n
\]
holds with probability at least $1-2e^{-n/c}$, which yields
\[
\mu_n({\Ab}_{-i} ) \ge \lambda + \frac{1}{c} \sum_{j > k, j \ne i} \lambda_i - c \lambda_{k+1} n
\ge \lambda + \frac{1}{2c} \sum_{j>k} \lambda_j - \rbr{ c + \frac{1}{c}} \lambda_{k+1} n,
\]
where the last inequality is because: 
(1) $\sum_{j>k, j\ne i} \lambda_j \ge \frac{1}{2}\sum_{j>k} \lambda_j$ if $i > k+1$,
and (2) $\sum_{j>k, j\ne i} \lambda_j = \sum_{j>k} \lambda_j - \lambda_{k+1}$ if $i=k+1$.
Finally, using the condition that $\rho_k \ge bn$ we obtain that for $i  > k$,
\[
\mu_n({\Ab}_{-i} ) \ge \lambda + \frac{1}{2c} \sum_{j>k} \lambda_j - (c + \frac{1}{c}) \lambda_{k+1} n
\ge\rbr{ \frac{1}{2c} - \frac{c}{b} - \frac{1}{cb} } \cdot \rbr{ \lambda + \sum_{j>k} \lambda_j },
\]
which completes the proof by letting $b > 4 c^2$ and $c \ge 1$

\end{proof}

\paragraph{Variance Lower Bounds.}
According to Lemma 7 in \citep{bartlett2020benign}, and note that $\bepsilon$ is independent of $\Xb$, has zero mean, and is $\sigma$-sub-Gaussian, we have that
\begin{equation}\label{eq:ridge-var-lb-trace}
    \EE_{\bepsilon} [\ridgevar] = \EE_{\bepsilon} [ \bepsilon^\top \Cb \bepsilon ] = \tr\rbr{\Cb \cdot \EE[\bepsilon \bepsilon^\top ]} 
    \ge \frac{1}{c} \sigma^2 \tr(\Cb)
\end{equation}
for some constant $c > 1$.
In the following we lower bound $\tr(\Cb)$.

\begin{lemma}[Variant of Lemma 8 in \citep{bartlett2020benign}]\label{lemma:ridge-var-lb-trace-decomp}
\begin{equation*}
    \tr(\Cb) = \sum_{i} \lambda_i^2 \zb_i^\top \Ab^{-2} \zb_i = \sum_{i} \frac{\lambda_i^2 \zb_i^\top \Ab_{-i}^{-2} \zb_i}{\rbr{1+\lambda_i \zb_i^\top \Ab_{-i}^{-1} \zb_i}^2}.
\end{equation*}
\end{lemma}
\begin{proof}
This is from the proof of Lemma 14 in \citep{tsigler2020benign}, and can be proved in the same way as Lemma 8 in \citep{bartlett2020benign}.
\end{proof}

\begin{lemma}[Variant of Lemma 14 in \citep{bartlett2020benign}]\label{lemma:ridge-var-lb-each-item}
There is a constant $c$ such that for any $i \ge 1$ with $\lambda_i > 0$, and any $0 \le k \le n/c$, with probability at least $0.1$, 
\begin{equation*}
    \frac{\lambda_i^2 \zb_i^\top \Ab_{-i}^{-2} \zb_i}{\rbr{1+\lambda_i \zb_i^\top \Ab_{-i}^{-1} \zb_i}^2} 
    \ge  \frac{1}{cn} \cdot \rbr{1 + \frac{\lambda_{k+1}}{\lambda_i} \cdot \rbr{1 + \frac{\rho_k}{n}}}^{-2}
\end{equation*}
\end{lemma}
\begin{proof}
Let $\cL_i$ be a random subspace if $\RR^n$ of codimension $k$, then
\begin{align*}
    \zb_i^\top \Ab_{-i}^{-1} \zb_i 
    &\ge \frac{1}{c_1} \cdot \frac{ \nbr{\Pi_{\cL_i} \zb_i}_2^2 }{{ \lambda + \sum_{j > k} \lambda_j +\lambda_{k+1} n }} \qquad (\text{by Lemma \ref{lemma:ridge-eigenvalue-concentration}}) \\
    &\ge \frac{1}{c_2}\cdot \frac{ n }{{ \lambda + \sum_{j > k} \lambda_j +\lambda_{k+1} n }} \qquad (\text{by Corollary 13 in \citep{bartlett2020benign}}) \\
    &= \frac{1}{c_2}\cdot \frac{ n }{{ \lambda_{k+1} (\rho_k + n) }},
\end{align*}
where $c_1, c_2 > 1$ are constants.
The above implies that 
\begin{align*}
    \frac{\lambda_i^2 \zb_i^\top \Ab_{-i}^{-2} \zb_i}{\rbr{1+\lambda_i \zb_i^\top \Ab_{-i}^{-1} \zb_i}^2} 
    &= {\rbr{1 + \rbr{\lambda_i \zb_i^\top \Ab_{-i}^{-1} \zb_i}^{-1}}^{-2}} \cdot \frac{\nbr{\zb_i^\top \Ab_{-i}^{-1}}^2_2 }{ \rbr{\zb_i^\top \Ab_{-i}^{-1} \zb_i}^2 } \\
    &\ge {\rbr{1 + \rbr{\lambda_i \zb_i^\top \Ab_{-i}^{-1} \zb_i}^{-1}}^{-2}} \cdot \frac{1}{\nbr{\zb_i}_2^2} \quad (\text{by Cauchy-Schwarz's inequality}) \\
    &\ge \rbr{ 1 + c_2\cdot \frac{{ \lambda_{k+1} (\rho_k + n) }}{ n \lambda_i} }^{-2} \cdot \frac{1}{\nbr{\zb_i}_2^2}. \ (\text{by the lower bound for $\zb_i^\top \Ab_{-i}^{-1} \zb_i $})
\end{align*}
According to Corollary 13 in \citep{bartlett2020benign}, there is constant $c_3 > 1$ such that $\nbr{\zb_i}_2^2 \le \frac{1}{c_3} n$ holds with constant probability, inserting which into the above inequality and rescaling the constants complete the proof.
\end{proof}

\begin{lemma}[Variant of Lemma 16 in \citep{bartlett2020benign}]\label{lemma:ridge-var-lb-trace-lb}
There is constant $c$ such that for any $0 \le k \le n/c$ and any $b >1$ with probability at least $0.1$,
\begin{itemize}
    \item if $\rho_k < b n$, then $\tr(\Cb) \ge \frac{k+1}{c b^2 n}$;
    \item if $\rho_k \ge b n$, then $\tr(\Cb) \ge \frac{1}{c b^2} \min_{\ell \le k} \left\{ {\frac{\ell}{n}} + \frac{b^2 n\sum_{i>\ell} \lambda_i^2}{(\lambda_{k+1} \rho_k )^2} \right\}$.
\end{itemize}
\end{lemma}
\begin{proof}
This is proved by repeating the proof of Lemma 16 in \citep{bartlett2020benign}, where we replace Lemmas 8 and 14 in \citep{bartlett2020benign} with our Lemmas \ref{lemma:ridge-var-lb-trace-decomp} and \ref{lemma:ridge-var-lb-each-item} respectively.
\end{proof}

\begin{theorem}[Restatement of Theorem \ref{thm:lowerbound_ridge}, variance part]
There exist absolute constants $b, c, c_1 > 1$ for the following to hold:
let 
\[
k^* := \min \{k: \lambda + \sum_{i>k}\lambda_i \ge bn \lambda_{k+1} \},
\]
then with probability at least $0.1$:
\begin{itemize}
    \item if $k^* \ge n / c_1 $ then 
    \[
    \EE_{\bepsilon}[ \ridgevar] \ge \frac{\sigma^2}{c};\]
    \item if $k^* < n /c_1$ then 
    \[
    \EE_{\bepsilon}[ \ridgevar] \ge \frac{\sigma^2}{c} \rbr{\frac{k^*}{n} + \frac{n}{\lambda + \sum_{i > k^*} \lambda_i}\cdot  \sum_{i > k^*}  \lambda_i^2}.\]
\end{itemize}
As a direct consequence, the expected ridge variance is lower bounded by 
\[
\EE_{\Xb, \bepsilon}[\ridgevar] \ge 
\begin{cases}
    \frac{\sigma^2}{10c}, & k^* \ge n / c_1 \\
    \frac{\sigma^2}{10c} \rbr{\frac{k^*}{n} + \frac{n}{\lambda + \sum_{i > k^*} \lambda_i}\cdot  \sum_{i > k^*}  \lambda_i^2}, & k^* < n /c_1.
\end{cases}
\]

\end{theorem}
\begin{proof}
The high probability lower bound is proved by \eqref{eq:ridge-var-lb-trace}, our Lemma \ref{lemma:ridge-var-lb-trace-lb}, and Lemma 17 in \citep{bartlett2020benign}.
The expectation lower bound follows immediately from the high probability lower bound by noticing the ridge variance error is non-negative. 
\end{proof}

\paragraph{Bias Lower Bound.}

Recall the ridge bias error is \citep{tsigler2020benign}
\begin{align}
   \ridgebias  = (\wb^*)^\top \Bb \wb^* = \sum_{i} \rbr{\Bb}_{ii} (\wb^*_i)^2 + 2\sum_{i > j} \rbr{\Bb}_{ij} \wb^*_i \wb^*_j. \label{eq:ridge-bias-diag}
\end{align}
The following lemma shows the crossing terms are zero in expectation.
\begin{lemma}\label{lemma:ridge-bias-lb-crossing-item}
For $i \ne j$, 
\[\EE_{\Xb} [(\Bb)_{ij}] = 0.\]
\end{lemma}
\begin{proof}
Recall that 
\[ \Bb := \rbr{\Ib_d - \Xb^\top  \Ab^{-1} \Xb} \Hb  \rbr{\Ib_d - \Xb^\top  \Ab^{-1} \Xb}.\]
Recall that 
$\Xb = \rbr{ \sqrt{\lambda_1} \zb_1,  \dots \sqrt{\lambda_d} \zb_d }$, thus the $i$-th column of $\rbr{\Ib_d - \Xb^\top  \Ab^{-1} \Xb}$ is 
\[
\rbr{\Ib_d - \Xb^\top  \Ab^{-1} \Xb}_i
= \eb_i - \sqrt{\lambda_i} \Xb^\top \Ab^{-1} \zb_i.
\]
Moreover recall $\Hb = \diag (\lambda_1, \dots, \lambda_d)$, therefore 
\begin{align*}
    (\Bb)_{ij} &= \eb_i^\top \Bb \eb_j 
    = \rbr{ \eb_i - \sqrt{\lambda_i} \Xb^\top \Ab^{-1} \zb_i}^\top  \Hb \rbr{ \eb_j - \sqrt{\lambda_j}  \Xb^\top \Ab^{-1} \zb_j}  \\
    &= \eb_i^\top \Hb \eb_j - \sqrt{\lambda_i} \eb_j^\top \Hb \Xb^\top \Ab^{-1} \zb_i
    - \sqrt{\lambda_j} \eb_i^\top \Hb  \Xb^\top \Ab^{-1} \zb_j 
    + \sqrt{\lambda_i\lambda_j} \zb_i^\top \Ab^{-1} \Xb \Hb  \Xb^\top \Ab^{-1} \zb_j \\
    &= \eb_i^\top \Hb \eb_j - \rbr{\sqrt{\lambda_i\lambda_j}\lambda_j + \sqrt{\lambda_i\lambda_j}\lambda_i} \zb_i^\top \Ab^{-1} \zb_j
    + \sqrt{\lambda_i\lambda_j} \zb_i^\top \Ab^{-1} \Xb \Hb  \Xb^\top \Ab^{-1} \zb_j.
\end{align*}
The first term is zero since $\Hb$ is diagonal and $i \ne j$.
We next show the second term is zero in expectation.
Indeed, let
\[
F(\zb_i) :=
\zb_i^\top \Ab^{-1} \zb_j = \zb_i^\top \rbr{ \Ab_{-i} + \lambda_i \zb_i \zb_i^\top }^{-1} \zb_j,
\]
where $\Ab_{-i}$ is independent of $\zb_i$, then $F(\zb_i) = - F(-\zb_i)$.
Also note that $\zb_i$ follows a standard Gaussian which is symmetric, therefore $\EE_{\zb_i} F(\zb_i) = 0$.
In a similar manner, the third term is also zero in expectation. The proof is then completed.

\end{proof}

\begin{lemma}[Part of the proof of Lemma 15 in \citep{tsigler2020benign}]\label{lemma:ridge-bias-lb-diagonal-item}
There exists absolute constant $c>1$, such that with probability at least $0.1$,
\[
\rbr{\Bb}_{ii} \ge \frac{1}{c}\cdot \frac{\lambda_i}{\rbr{1 + \frac{\lambda_i}{\lambda_{k+1}} \cdot \frac{n}{\rho_k}}^2}.
\]
As a direct consequence, 
\[
\EE_{\Xb} [ \rbr{\Bb}_{ii} ] \ge \frac{1}{10c}\cdot \frac{\lambda_i}{\rbr{1 + \frac{\lambda_i}{\lambda_{k+1}} \cdot \frac{n}{\rho_k}}^2}.
\]
\end{lemma}
\begin{proof}
This lemma summarizes part of the proof of Lemma 15 in \citep{tsigler2020benign}.
Recall that $\Hb$ is diagonal and 
$\Bb := \rbr{\Ib_d - \Xb^\top  \Ab^{-1} \Xb} \Hb  \rbr{\Ib_d - \Xb^\top  \Ab^{-1} \Xb}$, thus 
\begin{align*}
    \rbr{\Bb}_{ii} 
    &= \lambda_i \nbr{ \rbr{\Ib_d - \Xb^\top  \Ab^{-1} \Xb}_i }_2^2 \qquad (\text{since $\Hb$ is diagonal}) \\
    &= \lambda_i \nbr{ e_i^\top - \sqrt{\lambda_i} \zb_i^\top \Ab^{-1} \Xb }_2^2 \qquad (\text{$\Xb = \rbr{ \sqrt{\lambda_1} \zb_1, \dots, \sqrt{\lambda_j} \zb_j, \dots \sqrt{\lambda_d} \zb_d }$}) \\
    &= \lambda_i \nbr{ e_i^\top - \rbr{ \sqrt{\lambda_i \lambda_1} \zb_i^\top  \Ab^{-1} \zb_1 , \dots, \sqrt{\lambda_i \lambda_j} \zb_i^\top  \Ab^{-1} \zb_j, \dots \sqrt{\lambda_i \lambda_d} \zb_i^\top  \Ab^{-1} \zb_d  } }_2^2 \\
    &\ge \lambda_i \rbr{ 1- \lambda_i \zb_i^\top \Ab^{-1} \zb_i}^2 \qquad (\text{use Pythagorean theorem}) \\
    &= \frac{\lambda_i}{\rbr{1 + \lambda_i \zb_i^\top \Ab_{-i}^{-1} \zb_i}^2},
\end{align*}
where in the last step we use $\Ab = \Ab_{-i} + \lambda_i \zb_i \zb_i^\top$ and that
\begin{align*}
     1- \lambda_i \zb_i^\top \Ab^{-1} \zb_i
    &= 1- \lambda_i \zb_i^\top \rbr{\Ab_{-i} + \lambda_i \zb_i \zb_i^\top}^{-1} \zb_i \\
    &= 1- \lambda_i \zb_i^\top \rbr{\Ab_{-i}^{-1} - \lambda_i \Ab_{-i}^{-1} \zb_i (1+\zb_i^\top \Ab_{-i}^{-1}\zb_i)^{-1} \zb_i^\top \Ab_{-i}^{-1} } \zb_i \\
    &= \frac{1}{1 + \lambda_i \zb_i^\top \Ab_{-i}^{-1} \zb_i}.
\end{align*}

Now according to Corollary 13 in \citep{bartlett2020benign}, there exists constant $c_1 > 1$ such that 
\[
\nbr{\zb_i}_2^2 \le c_1 n
\]
holds with constant probability;
and according to Lemma \ref{lemma:ridge-eigenvalue-concentration}, there exists constant $c_2 > 1$ such that for any $i \ge 1$, 
\[
\mu_n(\Ab_{-i}) \ge \frac{1}{c_2} \lambda_{k+1} \rho_k
\]
holds with constant probability.
These two facts imply that 
\begin{align*}
    \zb_i^\top \Ab^{-1}_{-i} \zb_i
    &\le \mu_n(\Ab_{-i})^{-1} \nbr{\zb_i}_2^2 \le c_1c_2 \frac{n}{\lambda_{k+1} \rho_k},
\end{align*}
inserting which into the bound of $(\Bb)_{ii}$, we conclude that with constant probability,
\begin{align*}
    (\Bb)_{ii} &\ge \frac{\lambda_i}{\rbr{1 + \lambda_i \zb_i^\top \Ab_{-i}^{-1} \zb_i}^2} 
    \ge \frac{\lambda_i}{\rbr{1 + c_1 c_2 \cdot \frac{\lambda_i}{\lambda_{k+1}} \cdot \frac{n}{\rho_k}}^2}.
\end{align*}
Finally a rescaling of the constants completes the proof.
\end{proof}

\begin{theorem}[Restatement of Theorem \ref{thm:lowerbound_ridge}, bias part]
There exist absolute constants $b, c > 1$ for the following to hold:
let 
\[
k^* := \min \{k: \lambda + \sum_{i>k}\lambda_i \ge bn \lambda_{k+1} \},
\]
then
\[
\EE_{\Xb} [\ridgebias] \ge \frac{1}{c}\rbr{\frac{\lambda+\sum_{i> k^*} \lambda_i}{n^2}\cdot \nbr{\wb^*}_{\Hb_{0:k^*}^{-1}}^2 + \nbr{\wb^*}_{\Hb_{k^*:\infty}}^2   }.
\]
\end{theorem}
\begin{proof}
By \eqref{eq:ridge-bias-diag}, 
Lemmas \ref{lemma:ridge-bias-lb-crossing-item} and \ref{lemma:ridge-bias-lb-diagonal-item}, we have that, 
\begin{align*}
     \EE_{\Xb} [\ridgebias] & = 
     \sum_{i} \rbr{\Bb}_{ii} (\wb^*_i)^2 \\
     &\ge \frac{1}{c_1} \sum_{i} \frac{1}{\rbr{1 + \frac{\lambda_i}{\lambda_{k^*+1}} \cdot \frac{n}{\rho_{k^*}}}^2 } \cdot \lambda_i (\wb^*_i)^2 \qquad (\text{choose $k=k^*$})\\
      &\ge \frac{1}{c_1b^2} \sum_{i} \frac{1}{ \rbr{ \frac{1}{b} + \frac{\lambda_i}{\lambda_{k^*+1}} \cdot \frac{n}{\rho_{k^*}}}^2 } \cdot \lambda_i (\wb^*_i)^2,
\end{align*}
where $c_1, b > 1$ are all absolute constants.
Note that for all $i \le k^*$, we must have $\lambda + \sum_{j > i-1} \lambda_j < bn \lambda_i$, 
\begin{align*}
    \frac{\lambda_i}{\lambda_{k^*+1}} \cdot \frac{n}{\rho_{k^*}}
    = \frac{\lambda_i n}{\lambda + \sum_{j > k^*} \lambda_j} 
    \ge \frac{\lambda_i n}{\lambda + \sum_{j > i-1} \lambda_j}
    \ge  \frac{1}{b},
\end{align*}
and for all $i \ge k^* + 1$, we have 
\begin{align*}
    \frac{\lambda_i}{\lambda_{k^*+1}} \cdot \frac{n}{\rho_{k^*}}
    \le \frac{n}{\rho_{k^*}} \le \frac{1}{b},
\end{align*}
then 
\begin{align*}
    \EE_{\Xb} [\ridgebias] 
    &\ge \frac{1}{c_1 b^2} \sum_{i} \frac{1}{\rbr{ \frac{1}{b} + \frac{\lambda_i}{\lambda_{k^*+1}} \cdot \frac{n}{\rho_{k^*}}}^2 } \cdot \lambda_i (\wb^*_i)^2 \\
     &\ge \frac{1}{2c_1 b^2} \cdot \rbr{ \sum_{i\le k^*} \frac{1}{ \rbr{ \frac{\lambda_i}{\lambda_{k+1}} \cdot \frac{n}{\rho_k}}^2 } \cdot \lambda_i (\wb^*_i)^2 
     + \sum_{i> k^*} \frac{1}{\rbr{ 1/b}^2 } \cdot \lambda_i (\wb^*_i)^2  } \\
     &\ge \frac{1}{c }\rbr{ \sum_{i\le k^*}  \frac{\rbr{ \lambda_{k+1} \rho_k}^2 }{n^2} \cdot \lambda_i^{-1} (\wb^*_i)^2 
     + \sum_{i> k^*}  \lambda_i (\wb^*_i)^2  } \\
     &= \frac{1}{c }\rbr{ \frac{\lambda+\sum_{i> k^*} \lambda_i}{n^2}\cdot \nbr{\wb^*}_{\Hb_{0:k^*}^{-1}}^2 + \nbr{\wb^*}_{\Hb_{k^*:\infty}}^2 }, 
\end{align*}
where $c>1$ is an absolute constant.
\end{proof}

\end{document}